\documentclass[10pt]{article} %
\usepackage{arxiv}
\usepackage[numbers,sort&compress]{natbib}

\usepackage[utf8]{inputenc}
\usepackage[T1]{fontenc}
\usepackage{lmodern}      
\usepackage{microtype}    
\usepackage[english]{babel}

\usepackage{url}            
\usepackage{booktabs}       
\usepackage{amsfonts}       
\usepackage{nicefrac}       
\usepackage{xcolor}         
\usepackage{graphicx,amssymb,amsmath,amsthm}
\usepackage{microtype}
\usepackage{subcaption}
\usepackage[algo2e,linesnumbered,ruled,vlined]{algorithm2e}
\usepackage{varwidth}
\usepackage{enumitem}

\usepackage{tikz}
\usepackage{pgfplots}
\usetikzlibrary{pgfplots.groupplots,arrows,automata,positioning,shapes.multipart,positioning,arrows,calc}
\usepgfplotslibrary{dateplot,fillbetween,colormaps,colorbrewer}
\pgfplotsset{compat=1.17}
\usepackage{nicematrix}
\usepackage{booktabs}
\usepackage{caption}

\SetKwInput{KwInput}{Input}
\SetKwInput{KwOutput}{Output}

\newcommand\extrafootertext[1]{%
    \bgroup
    \renewcommand\thefootnote{\fnsymbol{footnote}}%
    \renewcommand\thempfootnote{\fnsymbol{mpfootnote}}%
    \footnotetext[0]{#1}%
    \egroup
}

\newcommand*{\R}{\mathbb{R}}
\newcommand*{\N}{\mathbb{N}}

\newtheorem{assumption}{Assumption}

\newtheorem{theorem}{Theorem}
\newtheorem{corollary}{Corollary}
\newtheorem{lemma}{Lemma}

\newtheorem{remark}{Remark}

\newcommand{\eps}{\varepsilon}

\newcommand{\set}[1]{\{#1\}}

\newcommand{\st}{\text{s.t.}}
\newcommand{\abs}[1]{\lvert{#1}\rvert}
\newcommand{\norm}[1]{\lVert{#1}\rVert}
\DeclareMathOperator*{\argmin}{argmin}
\DeclareMathOperator{\dist}{dist}

\newcommand{\nx}{m}
\newcommand{\ny}{p}
\newcommand{\nz}{d}

\newcommand{\x}{\lambda}
\newcommand{\y}{\theta}
\newcommand{\z}{w}
\newcommand{\X}{\Lambda}
\newcommand{\Y}{\Theta}
\newcommand{\bin}{\mathrm{bin}}
\newcommand{\Ybin}{\Theta_{\bin}}
\newcommand{\W}{\Lambda\times\Theta_{\bin}}

\newcommand{\define}{\mathrel{{\mathop:}{=}}}

\title{Relax and penalize: a new bilevel approach to \\mixed-binary hyperparameter optimization}

\author{Sara Venturini \\
      MIT Senseable City Lab \\
      Massachusetts Institute of Technology
      \And
    Marianna De Santis\\
      Department of Information Engineering \\
      University of Florence
      \And
      Jordan Patracone\\
      Inria, Laboratoire Hubert Curien \\
      Université Jean Monnet Saint-Etienne$^{\dagger}$
      \And
      Martin Schmidt\\
      Department of Mathematics\\
      Trier University 
      \AND
      Francesco Rinaldi\thanks{These authors contributed equally to this work.}\\
      Department of Mathematics \\
      University of Padova
      \And
      Saverio Salzo\footnotemark[1]\\
      DIAG \\
      Sapienza University of Rome
      }
      
\begin{document}

\maketitle

\begin{abstract}
In recent years, bilevel approaches have become very popular to
efficiently estimate high-dimensional hyperparameters of machine
learning models. However, to date, binary parameters are handled by
\emph{continuous relaxation and rounding} strategies, which could lead
to inconsistent solutions. In this context, we tackle the challenging
optimization of mixed-binary hyperparameters by resorting to an
equivalent continuous bilevel reformulation based on an appropriate
penalty term. We propose an algorithmic framework that, under suitable
assumptions, is guaranteed to provide mixed-binary
solutions. Moreover, the generality of the method allows to safely use
existing continuous bilevel solvers within the proposed framework. We
evaluate the performance of our approach for two specific
machine learning problems, i.e., the estimation of the
group-sparsity structure in regression problems and the data
  distillation problem.
The reported results show that our method is competitive with state-of-the-art approaches based on relaxation and rounding.

\end{abstract}


\section{Introduction}
\extrafootertext{\!\!$^{\dagger}$Université Jean Monnet Saint-Etienne, CNRS, Institut d’Optique Graduate School, Inria, Laboratoire Hubert Curien UMR 5516, F-42023, SAINT- ETIENNE, France.}Nowadays, machine learning systems tend to incorporate an increasing
number of hyperparameters with the purpose of improving the overall
performance of learning tasks and achieving a higher flexibility.
Then, optimizing such high-dimensional hyperparameters becomes a
a crucial step for devising an efficient and fully parameter-free machine
learning systems.
In recent years, bilevel approaches to hyperparameter optimization have
become very popular as an effective way to estimate high-dimensional
hyperparameters~\citep{Arbel2021b,bae2020delta,bennett2006model,franceschi2018bilevel,grazzi20a,
  maclaurin15,pedregosa16}.
On the other hand, in many circumstances, binary hyperparameters are
included in the model to allow the pruning of the irrelevant variables
or the discovery of sparsity structures. Interesting examples are
given by the pruning of large-scale deep learning
models~\citep{zhangadvancing}, the identification of the
group-sparsity structures in regression problems
\citep{2018_Frecon_J_p-neurips_biglasso,Wang2020},
and learning the discrete structure of a graph neural networks~\citep{franceschi19a}.
For these cases, the usual optimization approach is that of
\emph{relaxing} the respective parameterIn
over the unit interval $[0,1]$, solve the continuous optimization problem, and then \emph{rounding} the solution so to get a binary output.
This is essentially a heuristic, which overcomes the challenge of dealing with integer variables, but
does not offer any theoretical guarantees.
The aim of the present work is to provide a more principled way of approaching mixed-binary hyperparameter optimization.

\textbf{Related works.}
In the context of machine learning,
bilevel optimization problems with binary variables arise in a number of situations.
In~\citet{2018_Frecon_J_p-neurips_biglasso}
the estimation of group-sparsity structures
in multi-task regression is addressed by
a mixed-binary bilevel optimization model,
which is handled by a continuous relaxation and approximation of the problem. The output of the optimization procedure is a vector of continuous variables, which are then rounded to the closest binary values,
so to provide the final grouping of the features. In Section~\ref{sec:applications}, we
tackle this same problem and show the advantage of our approach.
In the work~\citet{zhangadvancing},
a new model pruning, based
on bilevel optimization, is proposed, where the upper level variable is a binary mask. The related iterative algorithm performs a gradient descent step on the continuous relaxation of the problem followed by a projection step onto a discrete set, which is indeed a hard-thresholding operation. No convergence guarantees are provided. 
In~\citet{borsosetal2024}, the authors present a general framework for coreset construction by formulating coreset selection as a cardinality-constrained bilevel optimization problem, solved using a tailored algorithm that combines greedy forward
selection and first-order methods. The proposed approach is model-agnostic and applicable to any twice-differentiable model, including neural networks. A drawback is that the coreset weights must be determined for each selection step, which involves an iterative process. To streamline this, binary weights (i.e., unweighted coreset) are also used and a mixed-binary bilevel optimization problem is defined, thus eliminating the weight optimization step. The authors only report a theoretical analysis of the algorithm for the continuous-weights case.
In another recent paper~\citet{zhou2024learning}
some deep learning techniques are developed to tackle a bilevel problem with a binary tender, i.e., a problem where the upper and lower levels are connected through binary variables. A neural network is trained to approximate the optimal value of the lower-level problem as a function of the binary tender. This enables a single-level reformulation of the bilevel program using a mixed-integer representation of the value function. Additionally, a comparative analysis is conducted between two neural network architectures—general neural networks and novel input-supermodular neural networks—to assess their representational capacities. To handle high-dimensional bilevel programs, an enhanced sampling method is introduced to generate higher-quality samples, along with an iterative process to refine solutions.

Beyond the machine learning literature, there are a number of works
related to mixed-integer bilevel programming; see, e.g., Section~5.3
of the recent survey~\citet{Kleinert_et_al:2021c}. An important
drawback one needs to take into account when applying those methods to
hyperparameter optimization problems are however limited
scalability.
Common approaches like, e.g.,  the outer-approximation-based method
in~\citet{Kleinert_et_al:2021b}, or the algorithm proposed
in~\citet{Mitsos:2010}, which requires the global solution of a
a significant number of mixed-integer nonlinear programs have indeed a
prohibitive cost when the dimensionality grows. Furthermore, it
should be noted that they aim to achieve global optimality---an
overly ambitious goal in the context of machine learning
applications.




\textbf{Contributions and outline.}
In this paper, we analyze more carefully the mixed-binary setting and
propose a \emph{relax and penalize} method, which produces a
mixed-binary output and relies on improved mathematical grounds. More
precisely, we present in Section~\ref{sec:problem-statement} a general
mixed continuous-binary bilevel problem and show in
Section~\ref{sec:relpenprob} that it is equivalent, in terms of global
minima and minimizers, to a fully continuous and penalized
optimization problem.
Next, in Section~\ref{sec:penalty-method}, we propose an algorithmic
framework, which consists of iteratively solving a sequence of
continuous and penalized problems which, under suitable assumptions, are guaranteed to provide mixed-binary local solutions.
The performance of the proposed approach is quantitatively assessed on the two machine-learning applications of the group structure estimation in the group lasso problem and the data distillation task.
 Numerical experiments are reported in Section~\ref{sec:applications} and show how the \emph{relax and penalize} method compares with state-of-the-art approaches based on relaxation and rounding. 
Finally, conclusions and perspectives are drawn in Section~\ref{sec:conc}.

\textbf{Notation.}
For every integer $n \geq 1$, $[n]$ denotes the set $\{1, \dots, n\}$.
We denote by $\norm{\cdot}$ the Euclidean norm in $\R^n$ and by
$\norm{\cdot}_{\infty}$ the infinity norm, meaning $\norm{x}_{\infty}
= \max_{1 \leq u \leq n} \abs{x_i}$. If $x \in \R^n$
and $\rho>0$ we denote by $B_{\rho}(x)$ the closed ball in $\R^n$ with
center~$x$ and radius~$\rho$,
i.e., $B_\rho(x) = \{x^\prime \in \R^n \colon \norm{x^\prime - x}
\leq \rho\}$. The standard $(n-1)$-simplex is denoted by $\Delta^{n-1} = \{x \in \R_+^n \colon \sum_{i=1}^n x_i = 1\}$.
Moreover, $\odot$ is the Hadamard product, meaning the component-wise
multiplication of vectors in $\R^d$. If $\Psi\colon \R^n \to \R$
is a continuous function and $\Omega \subseteq \R^n$, we denote by
$\argmin_{\Omega} \Psi$ the set of minimizers of $\Psi$
over $\Omega$ and with a slight abuse of notation also
the minimizer itself when it is unique.


\section{Problem statement}
\label{sec:problem-statement}

We consider mixed-binary bilevel problems of the form
\begin{subequations}
  \label{prob:orig}
  \begin{align}
    \min_{\x,\y} \quad
    & F(\x,\y,\z(\x,\y))
    \\
    \st \quad
    & \x \in \X\subseteq \R^{\nx}, \ \y \in \Y_{\mathrm{bin}}\subseteq \set{0,1}^{\ny},
    \\
    & \z(\x,\y) = \argmin_{\z \in W(\x,\y)} f(\x,\y,\z),
  \end{align}
\end{subequations}
where $F,f\colon\R^{\nx}\times\R^{\ny} \times \R^{\nz} \to \R$ and $W(\x,\y)\subseteq \R^{\nz}$.
We note that the lower-level problem is supposed to
admit a unique solution and that the hyperparameters $\x$ and $\y$
are continuous and binary variables, respectively. In the context of
machine learning problems, the functions $F$ and $f$ often are the
loss over a validation set and training set, respectively. We will
provide major applications of this situation in
Section~\ref{sec:applications}.

In the remainder of this paper we assume that the binary set~$\Ybin$ is
embedded in a larger continuous set~$\Y$
and that the lower-level problem admits a unique solution also for
$\y \in \Y$.
Then, we can consider the following and more compact formulation
\begin{equation}
  \label{prob:origG}
  \min_{(\x,\y) \in \X\times\Y_{\bin}} G(\x,\y)
\end{equation}
of the above problem, where we set $G(\x,\y)\define
F(\x,\y,\z(\x,\y))$ and require that the following assumption holds.

\begin{assumption}
\label{ass:a1}
\
\begin{enumerate}[label={\rm (\roman*)}]
\item $\X\subseteq \R^{\nx}$ is nonempty and compact.
\item\label{ass:theta} $\Y_{\bin} \define  \Y \cap \{0,1\}^{\ny}\neq \varnothing$,
    where $\Y \subseteq [0,1]^{\ny}$ is convex and compact and $\Y\setminus\Y_{
      \bin} \neq \varnothing$.
\item\label{ass:a1_iii} $G\colon \X \times \Y\to \R$ is continuous.
\item\label{ass:a1_iv} \label{ass:lip} For all $\x \in \X$, the map $G(\x,\cdot)$ is Lipschitz continuous with constant $L>0$ on $\Y$.
\end{enumerate}
\end{assumption}

Assumption~\ref{ass:a1} can be met with appropriate hypotheses on the functions $F$
and $f$. For instance, a sufficient condition for \ref{ass:a1_iii} is that
the functions $F$ and $f$ are jointly continuous, that the function $f
(\x,\y,\cdot)$ is strongly convex with a modulus of convexity which is
uniform for every $(\x,\y)$, and that the set-valued mapping $(\x,\y)
\mapsto W(\x,\y)$ is closed and such that, for every $(\bar{\x},
\bar{\y}) \in \X\times\Y$, $\mathrm{dist}(\z(\bar{\x},\bar{\y}),
W(\z,\y))\to 0$ as $(\x,y)\to (\bar{\x}, \bar{\y})$
\citep[Proposition~4.4]{bonnans2000}.
Additional conditions can ensure the validity of \ref{ass:a1_iv} too; see, 
\citet[Section~4.4]{bonnans2000}.



\section{Restating the problem via a smooth penalty function}
\label{sec:relpenprob}

In order to deal with the binary variables in
problem~\eqref{prob:origG}, we relax the integrality constraints on
$\y$ via a classic penalty term.
This leads to the continuous optimization problem
\begin{equation}
  \label{prob:ref}
    \min_{(\x,\y) \in \X\times\Y} G(\x,\y) + \frac{1}{\eps} \varphi(\y)
\end{equation}
in which we use the penalty function
\begin{equation}\label{eq:penfun}
  \varphi(\y) = \sum_{i=1}^{\ny} \y_i (1 - \y_i).
\end{equation}
Note that the function in~\eqref{eq:penfun} is a smooth, concave, and
quadratic function with the following properties:
\begin{equation*}
\forall\, \y \in [0,1]^{\ny}\colon \varphi(\y) \geq 0
\quad\text{and}\quad
\forall\, \y \in \{0,1\}^{\ny}\colon \varphi(\y) = 0.
\end{equation*}
This penalty has been introduced in~\citet{raghavachari1969connections} to
define equivalent continuous reformulations of mixed-integer linear
programming problems. In Section~\ref{sec:auxiliary} we give the main
properties of this penalty function that we use to prove the main
results of this and the next section.

We start with a result establishing the equivalence of
Problems~\eqref{prob:origG} and \eqref{prob:ref} in terms of
global minimizers. It is in line with a stream of works analyzing
the use of concave penalty functions in the framework of nonlinear
optimization problems with binary or integer variables (see, e.g.,
\citet{giannessi1976connections,kalantari1987penalty,Lucidi_2010}
and references therein). For the reader's convenience we provide the proof
of this result in the appendix.

\begin{theorem}\label{teo:equiv}
  Suppose that
  Assumption~\ref{ass:a1} is satisfied.
  Then, there exists an $\bar{\varepsilon} > 0$ such that
  for all $\varepsilon \in \left]0,\bar{\varepsilon}\right]$,
  Problems~\eqref{prob:origG} and \eqref{prob:ref} have the same global
  minimizers, i.e.,
  \begin{equation*}
  \argmin\limits_{(\x,\y) \in \W} G(\x,\y)
  = \argmin_{(\x,\y) \in \X \times \Y}G(\x,\y)+\frac{1}{\varepsilon} \varphi(\y).
\end{equation*}
\end{theorem}
The conclusion of Theorem~\ref{teo:equiv} is remarkable since it
guarantees that despite the fact that \eqref{prob:ref} is a purely continuous
optimization problem, for $\varepsilon$ sufficiently small, all of its
global minimizers are mixed-binary feasible and are exactly the global
minimizers of the original problem~\eqref{prob:origG}.
\begin{remark}\
  \begin{enumerate}[label={\rm (\roman*)},leftmargin=4ex]
  \item\label{rmk1_i}
    Set $G_{\varepsilon}\colon \X\times \Y \to \R$ such that
    $G_\varepsilon(\x,\y) = G(\x,\y)+\varepsilon^{-1} \varphi(\y)$
    and let $\delta_{\X\times \Ybin}\colon \X\times\Y \to \R$ be the
    indicator function of the set $\X\times \Ybin$, i.e.,
    the function that is zero on $\X\times \Ybin$ and $+\infty$
    otherwise.
    Then, it is easy to see that $G_\varepsilon$
    $\Gamma$-converges\footnote{This type of convergence of functions is
      also known as epiconvergence.} to $G + \delta_{\X\times \Ybin}$ as
    $\varepsilon \to 0$.
    Moreover, the family of functions $(G_\varepsilon)_{\varepsilon>0}$
    is clearly equicoercive since they are all defined on the compact
    set~$\X\times \Y$. Therefore, it holds
    \begin{equation*}
      \argmin_{\X\times \Y} G_\varepsilon \to \argmin_{\X\times \Y} G
      + \delta_{\X\times \Ybin} = \argmin_{\X\times \Ybin} G
      \quad \text{as}\ \varepsilon\to 0
    \end{equation*}
    in the sense of set convergence. This is a standard result from variational
    analysis \citep{Dontchev93} and it is always true provided that $G$
    and $\varphi$ are continuous functions as well as that
    $\varphi \geq 0$ and  $\varphi(\y)=0$ if and only if $\y \in
    \Ybin$ holds.
  \item In view of \ref{rmk1_i}, which gives an asymptotic result, the
    statement of Theorem~\ref{teo:equiv} is stronger in the sense that, for the
    special function \eqref{eq:penfun} and for $\varepsilon$ small
    enough, $\argmin_{\X\times \Y} G_\varepsilon = \argmin_{\X\times
      \Ybin} G$ holds. 
  \end{enumerate}
\end{remark}

The previous theorem provides a justification to address
problem~\eqref{prob:ref} instead of \eqref{prob:origG}.
However, because the objective function in~\eqref{prob:ref} is
nonconvex, only local minimizers
are computationally approachable. 
%
Thus, the idea is that of looking for local minimizers of
\eqref{prob:ref} which are also mixed-binary---since the global
minimizers of \eqref{prob:origG} lie among them.

The next result is entirely new and addresses the issue of identifying
mixed-binary local minimizers of the objective in~\eqref{prob:ref},
providing a sufficient condition for that purpose.



\begin{theorem}\label{thm:01b}
Suppose that Assumption~\ref{ass:a1} holds.
  Let $c\in \, ]0, 1/2[$ and $0<\varepsilon < (1-2c)/L$.
  Moreover, let $(\bar \x, \bar \y)$ be a local minimizer of
  \begin{equation*}
G(\x,\y) + \frac{1}{\varepsilon} \varphi(\y) \;\mbox{ on }\; \Lambda\times\Theta.
\end{equation*}
  If  $\dist_{\infty}(\bar \y, \Ybin)\define \inf_{\y \in
    \Ybin} \norm{\bar{\y} - \y}_{\infty} < c$, then $\bar \y \in
  \Ybin$.
\end{theorem}
\begin{proof}
  Since $\dist_{\infty}(\bar \y, \Ybin) < c$,
  there exists $\y \in \Ybin$ such that  $\norm{\bar \y - \y}_{\infty} < c$.
  Let
  \begin{equation*}
    \y_t \define (1-t)\bar \y + t\y = \bar \y + t(\y-\bar \y)
    \quad\text{with}\ t\in [0,1].
  \end{equation*}
  In particular, $\norm{\y_t - \bar \y}_{\infty} = t\norm{\y-\bar \y}_{\infty} < tc \leq c$
  and $\y_t \in \Y$, since~$\Y$ is convex.
  By Lemma~\ref{lemma:varphi},
  \begin{equation}\label{eq:varphi}
    \varphi(\bar \y) - \varphi(\y_t) \geq (1-2c) \| \y_t-\bar \y\|
  \end{equation}
  holds. Moreover, there exists $\rho > 0$ such that for all $(\x',\y')  \in
  B_\rho (\bar \x, \bar \y) \cap (\Lambda\times\Theta)$, it holds
  \begin{equation*}
    G(\bar \x, \bar \y) + \frac 1 \varepsilon \varphi(\bar \y) \leq
    G(\x',\y') + \frac 1 \varepsilon \varphi(\y').
  \end{equation*}
  Now, take $t\in \, ]0,1[$ such that $t < \rho/(c \sqrt{\ny})$.
  Then, $\y_t \in \Y$ and $\norm{\y_t - \bar \y} \leq \sqrt{\ny}
    \norm{\y_t - \bar \y}_{\infty} <  \sqrt{\ny}\, t c < \rho$.
  Therefore, $(\bar \x, \y_t) \in B_\rho (\bar \x, \bar \y) \cap (\Lambda\times\Theta)$
  and we obtain
  \begin{align}
    \nonumber G(\bar \x, \y_t) - G(\bar \x, \bar \y) + \frac{1}{\varepsilon}\varphi(\y_t) - \frac{1}{\varepsilon}\varphi(\bar \y) & \leq L\| \y_t - \bar \y\| + \frac{1}{\varepsilon}(\varphi(\y_t) - \varphi(\bar \y))\\
    \nonumber  & \overset{\eqref{eq:varphi}}{\leq}  L\| \y_t - \bar \y\| - \frac{1-2c}{\varepsilon} \|\y_t - \bar \y\| \\
                                                                                                                                  & = \underbrace{\bigg(L - \frac{1-2c}{\varepsilon}\bigg)}_{< 0} \|\y_t - \bar \y\|. \label{eq:contr}
  \end{align}
  Moreover, if $\bar \y \not\in \{0,1\}^{\ny}$, since $\y\in \{0,1\}^{\ny}$, we have $\| \y - \bar \y\|>0$ and hence
  $\| \y_t - \bar \y\|>0$ for $t> 0$.
  Thus, \eqref{eq:contr} is strictly negative and it holds
  \begin{equation*}
    G(\bar \x, \y_t)  + \frac{1}{\varepsilon}\varphi(\y_t) < G(\bar \x,
    \bar \y)  + \frac{1}{\varepsilon}\varphi(\bar \y),
  \end{equation*}
  which gives a contradiction. Thus, necessarily $\bar \y \in \{0,1\}^{\ny}$.
\end{proof}

\begin{remark}\
\begin{enumerate}[label={\rm (\roman*)},leftmargin=4ex]
\item
Theorem~\ref{thm:01b} essentially says that, if $\varepsilon$ is small enough,
within the distance of $1/2$ measured with the infinity norm, there
are no other local minimizers of \eqref{prob:ref} than the ones that
are mixed-binary feasible.
\item Note that it does not make much sense to consider local
  minimizers of the function $G$ over $\X \times \Ybin$,
since any point in~$\Ybin$ is an isolated point and thus one can find
a corresponding local minimizer for each one of them.
\end{enumerate}
\end{remark}

\section{An iterative penalty method}
\label{sec:penalty-method}



We now present an iterative method addressing problem~\eqref{prob:origG}.
The idea is that
of solving a sequence of problems of
the form \eqref{prob:ref}, indexed with $k$, with decreasing parameters~$\eps_k$.
Hence, the problem to be solved in each iteration reads
\begin{equation}
  \label{prob:refk}
  \tag{\textnormal{P$^k$}}
    \min_{(\x,\y) \in \X\times\Y}  G(\x, \y) + \frac{1}{\eps^k} \varphi(\y).
\end{equation}
Then, thanks to Theorem~\ref{teo:equiv}, it is clear that  after a
finite number of iterations, the original mixed-binary optimization
problem and the relaxed and penalized one \eqref{prob:refk} become
equivalent in terms of global minimizers. Moreover, as we have already
discussed in the previous section, in practice we can only target the
computation of local minimizers, but we can restrict the search to the
mixed-binary ones. In the following, we make this strategy more
precise.

\begin{theorem}\label{teo:limitpointPk}
Suppose that Assumption~\ref{ass:a1} holds.
Let $(\varepsilon_k)_{k \in \N}$ be a vanishing sequence of positive numbers and, for every $k \in \N$, let $(\x^k,\y^k)$ be a local minimizer of \eqref{prob:refk}. Then,
\begin{equation*}
   \liminf\limits_{k\to+\infty} \dist_{\infty}(\y^k, \Ybin) < 1/2
   \implies \exists\, k \in \N\ \text{s.t.}\ \y^k \in \Ybin.
\end{equation*}
Moreover, if $\y^k \in \Ybin$, then
we have that $\x^k$ is a local minimizer of
\begin{equation*}
    \min_{\x \in \X} G(\x, \y^k).
\end{equation*}
\end{theorem}
\begin{proof}
Suppose that
$\liminf\limits_{k\to+\infty} \dist_{\infty}(\y^k, \Ybin) < 1/2$ and let $c>0$ such that $\liminf\limits_{k\to+\infty} \dist_{\infty}(\y^k, \Ybin) < c <1/2$. Then,
 there exists a subsequence $(\y^{n_k})_{k \in \N}$
 such that
\begin{equation*}
\forall\, k \in \N\colon\ \dist_{\infty}(\y^{n_k}, \Ybin)
    <c
    \quad\text{and}\quad \varepsilon^{n_k} \to 0.
\end{equation*}
Thus, there exists $k \in \N$ such that
\begin{equation*}
    \dist_{\infty}(\y^{n_k}, \Ybin)< c
    \quad\text{and}\quad \varepsilon_{n_k}< \frac{1-2c}{L}
\end{equation*}
and this, in view of Theorem~\ref{thm:01b},
gives that $\y^{n_k} \in \Ybin$.
Concerning the second part of the statement, suppose that $\y^k \in \Ybin$,
where $(\x^k, \y^k)$ is a local minimizer of  \eqref{prob:refk}.
Then, $\y^k \in \{0,1\}^{\ny}$ and there exists $\rho_k>0$
such that
\begin{equation*}
    \forall\, (\x,\y) \in B_{\rho_k}(\x^k,\y^k)\cap(\X\times\Y)\colon\
    G(\x^k,\y^k) + \frac{1}{\varepsilon_k}\varphi(\y^k) \leq G(\x,\y) + \frac{1}{\varepsilon_k} \varphi(\y).
\end{equation*}
Therefore, taking $\y=\y^k$ in the above inequality and noting that $\varphi(\y^k)=0$,
we have
\begin{equation*}
    \forall\, \x \in B_{\rho_k}(\x^k)\cap\X\colon\
    G(\x^k,\y^k) + \frac{1}{\varepsilon_k}\underbrace{\varphi(\y^k)}_{=0} \leq G(\x,\y^k) + \frac{1}{\varepsilon_k} \underbrace{\varphi(\y^k)}_{=0},
\end{equation*}
which shows that $\x^k$ is a local minimizer of $G(\cdot, \y^k)$ over $\X$.
\end{proof}

\begin{algorithm2e}[t]
  \DontPrintSemicolon

  \KwInput{Problem~\eqref{prob:origG}, $\eps^0 > 0$, $\beta \in \left]0,1\right[$.}

  \For{$k = 0, 1, 2, \dotsc$}{
    Let $(\x^k,\y^k)$ be a solution (either local or global) of
    problem~\eqref{prob:refk}.

    \uIf{$\y^k \notin \set{0,1}^{\ny}$}{
      update $\eps^{k+1} = \beta \eps^k$
    } \uElse {
      \Return $(\x^k, \y^k)$.
    }
  }

  \caption{Penalty method}
  \label{alg:penalty-algorithm}
\end{algorithm2e}

\begin{remark}\ \label{remark3}
  \begin{enumerate}[label={\rm (\roman*)}]
  \item In the experiments given in Section~\ref{sec:applications},
    we checked that the condition considered in
    Theorem~\ref{teo:limitpointPk} always occurs, meaning that the
    distance  $\dist_{\infty}(\y^k,\Ybin)$, where $\y^k$ was obtained
    by solving problem \eqref{prob:refk} via a gradient-based
    subroutine, remains well-below the threshold~$1/2$ for $k$ sufficiently large (See Appendix \ref{sec:app_D3} for a plot of the evolution along the iterations). 
  \item We note that there might indeed exist points $\y\in \Y$ such that
    $\dist_{\infty}(\y, \Ybin) > 1/2$.
    For instance, if we take the standard $(\ny-1)$-simplex
    \begin{equation*}
      \Y = \Delta^{\ny-1} = \bigg\{\y \in \R_+^{\ny} \,\colon \sum_{i=1}^{\ny} \y_i=1\bigg\},
    \end{equation*}
    we have that
    \begin{equation}
      \label{eq:20230509b}
      \big\{\y \in \Y\,\colon \dist_{\infty}(\y, \Ybin) \geq 1/2\big\}
      = \Delta^{\ny-1} \cap [0,1/2]^{\ny},
    \end{equation}
    which for $\ny=3$ is the full equilateral triangle with vertices
    $(e_1+e_2)/2, (e_1+e_3)/2$ and $(e_2+e_3)/2$, where the $e_i$'s
    are the vectors of the canonical basis of $\R^{\ny}$. In general,
    the set in \eqref{eq:20230509b} is a polytope of dimension $\ny-1$
    with $2\ny$ facets and $(\ny(\ny-1)/2)$ vertices.
  \end{enumerate}
\end{remark}

The method is formally given in
Algorithm~\ref{alg:penalty-algorithm}.



\section{Two machine-learning applications}
\label{sec:applications}

In this section, we present two machine-learning applications:
  estimating the group-sparsity structure in regression problems
  (Subsection \ref{sec:GLS}) and performing data distillation
  (Subsection \ref{sec:DHC}). Firstly, we present the problem setting
  and the bilevel formulation. Secondly, we show how the problem fits
  our mathematical formulation. Finally, we provide a comparison with  relaxation and rounding using two different rounding strategies: the \emph{simple rounding} baseline, which obtains the mixed-binary solution by rounding each  variable individually to 0 or 1, and the \emph{top-k hard thresholding baseline}, where the mixed-integer solution is obtained by performing a top-$k$ hard thresholding operation on the continuous solution \footnote{which can be equivalently defined as rounding up to $1$ the largest $k$ components of the vector and $0$ the other ones.} \citep{2018_Frecon_J_p-neurips_biglasso,zhangadvancing}.
The codes are available on the GitHub page: \url{https://github.com/saraventurini/Relax-and-penalize}

\subsection{Group lasso structure}
\label{sec:GLS}

Here, we present the application of estimating group-sparsity structures, which is useful in areas such as gene expression analysis. We follow the formulation, the optimization algorithm, and the experimental setup in~\citet{2018_Frecon_J_p-neurips_biglasso}. In particular, we extend the approach by optimizing over both the hyperparameters~$\theta$ and~$\lambda$, instead of determining $\lambda$ by cross-validation. We compare our \emph{relax and penalize} strategy with
the \emph{relaxation and rounding} proposed in
~\citet{2018_Frecon_J_p-neurips_biglasso}, using both \emph{simple} and \emph{top-$k$ hard thresholding} rounding.
The datasets used in the tests are challenging variants of the synthetic datasets referenced in~\citet{2018_Frecon_J_p-neurips_biglasso}, specifically designed to create classes of differing sizes.


\paragraph{Problem setting and formulation.}
Given an output vector $y \in \R^N$ and a design matrix $X \in \R^{N\times P}$,
the group lasso problem can be formulated as follows
\begin{equation*}
  \min_{\z \in \R^P} \frac{1}{2} \norm{X  \z - y}^2+ \lambda
  \sum_{l=1}^L \norm{\y_l \odot \z}_2,
\end{equation*}
where $\lambda>0$ is a regularization parameter and $\y_l$ is a binary
vector (with entries in $\{0, 1\}$) indicating the features (components) of
$\z$ belonging to the $l$th group, meaning $\mathcal{G}_l = \{i \in
[P]\colon \y_{i,l}=1\}$, where the vectors~$\y_l$ are thought as
columns of a $P \times L$ matrix.
In the case of nonoverlapping groups, it is assumed that $\sum_{l=1}^L
\y_{j,l}=1$ for every $j \in [P]$.
In the classic literature on the topic,
the groups are assumed to be known
a priori \citep{Yuan2006, Zhao2009}, but often in practice there is no
clue about the structure of the groups and the problem is to infer
this group structure from the data.
However, this amounts to estimating the binary variables $\y_l$'s,
which in general poses a challenge.

In view of the discussion above, in the related literature a common
approach is to relax the problem allowing the $\y_l$'s to vary in the
continuum $[0,1]^{P}$. This approach was followed in
~\citet{2018_Frecon_J_p-neurips_biglasso},
in which the following bilevel optimization problem is proposed
\begin{equation}
  \label{prob:grouplasso}
  \begin{aligned}
    &\min_{\y \in \Y} \frac 1 T \sum_{t=1}^T C_t(\z_t(\x,\y))
    \quad \text{with} \quad
    \Y = \left\{ \y \in [0,1]^{P \times L} \colon \sum\limits_{l=1}^L
    \y_l = {\bf 1}_P\right\} = (\Delta^{L-1})^P
    \\
    &\text{and}\quad
    \displaystyle\z(\x,\y) = \argmin\limits_{(\z_1, \dots, \z_T) \in \R^{d \times T}}
    \frac 1 T \sum_{t=1}^T \bigg(
    \dfrac{1}{2} \norm{X_t  \z_t - y_t}^2
    + \lambda \sum_{l=1}^L \norm{\y_l \odot \z_t}_2 + \frac{\eta}{2} \norm{\z_t}^2
    \bigg).
  \end{aligned}
\end{equation}
Here, $(X_t, y_t)_{1 \leq t \leq T}$ defines $T$ regression problems in which the
regressors share the same group-sparsity structure, $C_t$ is a smooth
cost function acting as a validation error for the $t$-th task, and
$\y_l$ are thought as columns of a $P\times L$~matrix.
Note that in this formulation, $\lambda$ is supposed to be fixed
(possibly determined by a cross-validation procedure).
The regularization terms $\eta/2 \norm{\z_t}^2$,
with $\eta \ll 1$, are added to ensure uniqueness of the solution to
the lower-level problem and to devise a dual algorithmic procedure
generating a sequence $(\z^{(q)}(\x,\y))_{q \in \N}$ with smooth
updates (w.r.t.\ $\y$) such that $\z^{(q)}(\x,\y) \to \z(\x,\y)$
uniformly on $\Y$ as $q \to +\infty$; see
\citet[Section~3.2]{2018_Frecon_J_p-neurips_biglasso}.
Ultimately, the groups are
estimated by solving the problem
\begin{equation*}
  \min_{\y \in \Y} \frac 1 T \sum_{t=1}^T C_t(\z^{(q)}(\x,\y)),
\end{equation*}
with $q$ large enough, and by appropriately thresholding (a
posterior) the solution $\y$ in order to recover binary variables
$\y_l$'s.

\paragraph*{Proposed method.}
Our general \emph{relax and penalize} approach, as described in Section~\ref{sec:penalty-method}, allows us to bypass the last thresholding step and directly address
the more challenging problem
\begin{equation}
  \label{eq:20230503a}
  \min_{(\x,\y) \in \X\times \Ybin} \frac 1 T \sum_{t=1}^T
  C_t(\z^{(q)}(\x,\y))
  \quad \text{with} \quad
  \begin{cases}
    \X = [\x_{\min},\x_{\max}]\quad\text{with}\;
    0<\x_{\min}<\x_{\max},\\[1ex]
    \Ybin  = \Big\{\y \in \{0,1\}^{P\times L} \colon \forall\, j \in
    [P]\ \sum\limits_{l=1}^L \y_{j,l} = 1\Big\},\
  \end{cases}
\end{equation}
Note that in~\citet{2018_Frecon_J_p-neurips_biglasso}, $\lambda$
  is supposed to be fixed (possibly determined by a cross-validation
  procedure). Instead, here we are optimizing w.r.t.\ both the
  hyperparameters $\theta$ and $\lambda$, obtaining a mixed-binary
  problem in the end. In Section~\ref{sec:GLS_auxiliary} we report the
  details of the extension.
Now, since $\z^{(q)}(\x,\y)$ is smooth w.r.t.~$\y$,
the objective in \eqref{eq:20230503a} satisfies Assumption~\ref{ass:a1} and hence, in view of Theorem~\ref{teo:equiv} and
Theorem~\ref{teo:limitpointPk}, we can consider the \emph{relaxed and
  penalized} version of Problem~\eqref{eq:20230503a}, that is
\begin{equation}
  \label{prob:approximate2}
  \min_{(\x,\y) \in \X\times \Y}  \frac 1 T
  \sum_{t=1}^T \bigg( C_t(w^{(q)}(\x,\y)) + \frac 1 \varepsilon \varphi(\y) \bigg),
\end{equation}
and state that, if $\varepsilon$ is small enough, the two problems
\eqref{eq:20230503a} and \eqref{prob:approximate2} share the same
global minimizers. By leveraging this
equivalence, we study in the next section the added benefits of the
proposed Algorithm~\ref{alg:penalty-algorithm}.

\begin{remark}
According to \citet[Theorem~3.1]{2018_Frecon_J_p-neurips_biglasso} we
have that $\z^{(q)}(\x,\y) \to \z(\x,\y)$ as $q \to +\infty$ uniformly
on $\X\times \Ybin$, so that, similarly to
\citet[Theorem~2.1]{2018_Frecon_J_p-neurips_biglasso}, one can prove
that Problem~\eqref{eq:20230503a} converges as $q\to +\infty$ to
the problem
\begin{equation*}
  \min_{(\x,\y) \in \X\times \Ybin} \frac 1 T \sum_{t=1}^T C_t(\z(\x,\y))
\end{equation*}
in terms of optimal values and sets of global minimizers.
This provides a justification for addressing Problem~\eqref{eq:20230503a}.
\end{remark}

\paragraph{Numerical experiments.}
\label{sec:GLS_exp}
\begin{figure*}[htb]
  \centering
\input{data/heatmap}

\begin{tikzpicture}
    \begin{groupplot}[
        group style={
            group name=my plots,
            group size=5 by 1, 
            xlabels at=edge bottom,
            ylabels at=edge left,
            horizontal sep=0.5cm, 
            vertical sep=2cm, 
        },
        width=4cm, 
        height=7cm, 
        xlabel=groups,
        ylabel=features,
        colormap={bluewhite}{color=(white) rgb255=(90,96,191)},
        title style={font=\small, align=center},
        enlargelimits=false,
        axis on top,
        point meta min=0,
        point meta max=1,
        ylabel near ticks,
        y dir = reverse,
    ]
    
    \nextgroupplot[
        title=$\theta^*$ oracle,
    ]
    \addplot [matrix plot*,point meta=explicit] file {thetastar.dat};
    
    \nextgroupplot[
        title=$\theta^{r}$,
        yticklabels={}, ylabel={}, 
    ]
    \addplot [matrix plot*,point meta=explicit] file {thetar.dat};
    
    \nextgroupplot[
        title=$\bar\theta^{s}$,
        yticklabels={}, ylabel={}, 
    ]
    \addplot [matrix plot*,point meta=explicit] file {thetas.dat};

    \nextgroupplot[
        title=$\bar\theta^{top}$,
        yticklabels={}, ylabel={}, 
    ]
    \addplot [matrix plot*,point meta=explicit] file {thetatop.dat};
    
    \nextgroupplot[
        title=$\theta^{p}$,
        yticklabels={}, 
        ylabel={}, 
        colorbar,
        colorbar style={
            title=$ $,
            yticklabel style={
                /pgf/number format/.cd,
                fixed,
                precision=1,
                fixed zerofill,
            },
        },
    ]
    \addplot [matrix plot*,point meta=explicit] file {thetap.dat};
    
    \end{groupplot}
\end{tikzpicture}
  \caption{ Example of oracle group structure with \textit{random}
      sizes and parameter $a=0.5$, and the  
      corrisponding $\theta$ obtained by the \emph{relaxation and rounding} before the rounding procedure (ref.\ as $r$), after both \emph{simple} (ref.\ as $s$) and \emph{top-$1$ hard thresholding} rounding (ref.\ as $top$), and the \emph{relax and penalize} (ref.\ as $p$) methods.}
  \label{fig:1}
\end{figure*}

\setlength{\tabcolsep}{2.5pt}
\begin{table}[t]
\caption{
Test errors (mean $\pm$ standard deviation), i.e., the value of the upper level functions $G$, for the \emph{relaxation and rounding} (ref.\ as $r$), with both \emph{simple} (ref.\ as $s$) and \emph{top-1 hard thresholding} rounding (ref.\ as $top$), and the \emph{relax and penalize} (ref.\ as $p$) methods, over $3$ runs of \textit{inequal} and \textit{random} group structures, for $a \in \{0.1, 0.3, 0.5\}$.}
\label{tab:table1}
\centering
\begin{scriptsize}
\begin{NiceTabular}{l ccc  ccc}
 & \multicolumn{3}{c}{INEQUAL} & \multicolumn{3}{c}{RANDOM}\\
\cmidrule(lr){2-4} \cmidrule(lr){5-7}
 & 0.1 & 0.3 & 0.5 & 0.1 & 0.3  & 0.5\\
\midrule
$G(\lambda^{r},\theta^\star)$ & 0.04 $\pm$ 0.00 & 0.05 $\pm$ 0.00 & 0.06 $\pm$ 0.00 & 0.06 $\pm$ 0.00 & 0.07 $\pm$ 0.00 & 0.09 $\pm$ 0.00 \\  
$G(\lambda^{r},\theta^{r})$ & 0.04 $\pm$ 0.00 & 0.05 $\pm$ 0.00 & 0.06 $\pm$ 0.00 & 0.05 $\pm$ 0.00 & 0.06 $\pm$ 0.00 & 0.07 $\pm$ 0.01 \\ 
\midrule
$G(\lambda^{r},\bar\theta^{s})$ & 0.06 $\pm$ 0.02 & 0.05 $\pm$ 0.00 & 0.06 $\pm$ 0.00 & 0.12 $\pm$ 0.05 & 0.09 $\pm$ 0.03 & 0.14 $\pm$ 0.06 \\
\midrule
$G(\lambda^{r},\bar\theta^{\text{top}})$ & \pmb{$0.04 \pm 0.00$} & \pmb{$0.05 \pm 0.00$} & \pmb{$0.06 \pm 0.00$} & \pmb{$0.05 \pm 0.00$} & \pmb{$0.06 \pm 0.01$} & \pmb{$0.07 \pm 0.01$} \\ 
\midrule
\midrule
$G(\lambda^{p},\theta^\star)$ & 0.04 $\pm$ 0.00 & 0.05 $\pm$ 0.00 & 0.06 $\pm$ 0.01 & 0.06 $\pm$ 0.00 & 0.07 $\pm$ 0.00 & 0.09 $\pm$ 0.00 \\ 
\midrule
$G(\lambda^{p},\theta^{p})$ & 0.05 $\pm$ 0.00 & \pmb{$0.05 \pm 0.00$} & \pmb{$0.06 \pm 0.01$} & \pmb{$0.05 \pm 0.00$} & \pmb{$0.06 \pm 0.01$} & \pmb{$0.07 \pm 0.01$} \\ 
\bottomrule
\end{NiceTabular}
\end{scriptsize}
\vskip-1em
\end{table}

\setlength{\tabcolsep}{2.5pt}
\begin{table}[t]
\caption{
Reconstruction errors (mean $\pm$ standard deviation), i.e., Frobenius norm of the difference of the oracle regressor and the obtained regressors, for the \emph{relaxation and rounding} (ref.\ as $r$), with both \emph{simple} (ref.\ as $s$) and \emph{top-1 hard thresholding} rounding (ref.\ as $top$), and the \emph{relax and penalize} (ref.\ as $p$) methods, over 3 runs of \textit{inequal} and \textit{random} group structures, for $a \in \{0.1, 0.3, 0.5\}$.}
\label{tab:table2}
\centering
\begin{scriptsize}
\begin{NiceTabular}{l ccc  ccc}
 & \multicolumn{3}{c}{INEQUAL} & \multicolumn{3}{c}{RANDOM}\\
\cmidrule(lr){2-4} \cmidrule(lr){5-7}
 & 0.1 & 0.3 & 0.5 & 0.1 & 0.3  & 0.5\\
\midrule
$\norm{w(\lambda^{r},\theta^{\star})- w^\star}_F$ & 4.20 $\pm$ 0.21 & 5.24 $\pm$ 0.19 & 4.67 $\pm$ 0.15 & 5.34 $\pm$ 0.18 & 5.96 $\pm$ 0.16 & 6.69 $\pm$ 0.15 \\
$\norm{w(\lambda^{r},\theta^{r}) - w^\star}_F$ & 4.73 $\pm$ 0.35 & 5.64 $\pm$ 0.05 & 5.13 $\pm$ 0.22 & 5.50 $\pm$ 0.24 & 6.10 $\pm$ 0.25 & 6.86 $\pm$ 0.19 \\ 
\midrule
$\norm{w(\lambda^{r},\bar \theta^{s}) - w^\star}_F$ & 7.93 $\pm$ 1.52 & 5.99 $\pm$ 1.17 & 6.63 $\pm$ 0.16 & 11.17 $\pm$ 2.51 & 9.17 $\pm$ 2.05 & 11.95 $\pm$ 2.72 \\
\midrule
$\norm{w(\lambda^{r},\bar \theta^{\text{top}}) - w^\star}_F$ & \pmb{$5.10 \pm 0.37$} & \pmb{$5.99 \pm 0.15$} & \pmb{$5.40 \pm 0.16$} & \pmb{$5.93 \pm 0.16$} & \pmb{$6.50 \pm 0.41$} & 7.18 $\pm$ 0.25 \\
\midrule\midrule
$\norm{w(\lambda^{p},\theta^{\star})- w^\star}_F$ & 4.20 $\pm$ 0.20 & 5.32 $\pm$ 0.20 & 4.69 $\pm$ 0.15 & 5.33 $\pm$ 0.16 & 5.92 $\pm$ 0.11 & 6.75 $\pm$ 0.12 \\
\midrule
$\norm{w(\lambda^{p},\theta^{p}) - w^\star}_F$ & 5.15 $\pm$ 0.29 & 6.04 $\pm$ 0.11 & 5.54 $\pm$ 0.18 & 5.94 $\pm$ 0.17 & 6.57 $\pm$ 0.39 & \pmb{$7.17 \pm 0.23$} \\
\bottomrule
\end{NiceTabular}
\end{scriptsize}
\vskip-1em
\end{table}

The experimental setting is similar to that of
~\citet{2018_Frecon_J_p-neurips_biglasso}. Indeed, we create synthetic
datasets where each task amounts to predict an oracle regressor
$\z^\star\in\mathbb{R}^P$ and an oracle group structure $\theta^\star \in
\mathbb{R}^{P\times L}$,
such that for all $j \in [P]$ and $l \in [L]$, $\theta^\star_{j,l} = 1$ if $j$ belongs to group $l$ and $0$ otherwise.
We consider $P=100$ features, $N=20$ observations, $T=500$ tasks, and $L=10$ oracle groups.
For every task~$t\in [T]$, we generate the oracle regressor~$w^\star_t$ and the data $(X_t, y_t)$ as follows.
The regressor~$w^\star_t$ is generated such that its values are
non-zero in one group chosen at random. In particular, we study the
datasets as these values belong to the interval $[-1,-a]\cup[a,1]$ as $0<a\leq1$ varies.
We point out that in~\citet{2018_Frecon_J_p-neurips_biglasso},
$w^\star_t$ are considered binary, and the smaller $a$  is, the more
difficult the problem becomes.
The design matrix $X_t\in\mathbb{R}^{N\times P}$ is randomly drawn
according to a standard normal distribution and then normalized
column-wise.
The output $y_t$ is such that $y_t\sim\mathcal{N}(X_t w_t^\star,
0.2\mathrm{I}_D)$.
Validation and test sets are generated similarly.
Unlike~\citet{2018_Frecon_J_p-neurips_biglasso}, we allow groups of completely random sizes, with no predefined criteria. In particular, we consider two settings: the
\textit{unequal} one, with half groups of size $5$ and half groups of
size $15$, and the \textit{random} one, where the size of the groups
are allowed to be different and are calculated as follow.
We generate $L$ normally distributed random values, apply the softmax operation to obtain $L$ percentages, and use these percentages to distribute the $P$ features across the different groups.
All the results are averaged over $5$ runs.
See Section \ref{sec:GLS_auxiliary} for more details.

We compare several methods, which produce 
the following quantities: $(\lambda^{r},\theta^{r})$
obtained by \emph{a pure relaxation} method, without any rounding,
$(\lambda^{r},\bar\theta^{s})$ obtained by a \emph{simple} rounding of the relaxed solution,  $(\lambda^{r},\bar\theta^{\text{top}})$ obtained by \emph{top-1 hard thresholding} the relaxed solution, meaning assigning each feature to its dominant group, and finally
$(\lambda^{p},\theta^{p})$ obtained by the proposed \emph{relax and penalize}
method. In Figure~\ref{fig:1} we first show an example of the different
group structures retrieved by the various methods compared to the oracle group structure $\theta^\star$,
for a \textit{random} synthetic dataset with $a=0.5$.
We can notice that using directly the \emph{relaxed} solution, some values are less
than $0.5$.
As a result, after a simple rounding, some rows may end up as zeros, leaving
those features unassigned to any group.
Then, we evaluate the methods using two performance measures, which are also evaluated at the oracle group structure $\theta^\star$. The related results are reported in Table~\ref{tab:table1} and Table~\ref{tab:table2}.
Notice that $\theta^\star$, $\bar\theta^{s}$, $\bar\theta^{\text{top}}$, $\theta^{p}$ are
binary instead $\theta^{r}$ might not be so.
Overall, these preliminary experiments on a group lasso problem show that the proposed method and the \emph{relaxation and rounding method} with top-$1$ hard thresholding perform essentially the same, and they are both preferable to the relaxed and simple rounding procedure. We believe that 
the effectiveness of the top-$1$ hard thresholding is due the the fact that we are considering nonoverlapping group structures which put strong constraints on the binary variables, since it is all about identifying one nonzero entry for each feature. Finally, it is worth noting that the group lasso problems used in the study are  
randomly generated synthetic problems. As such, they appear to be less  
challenging compared to the real-world problems addressed in the next experiment.

\subsection{Data distillation}
\label{sec:DHC}

We now present the application of data distillation.
Data distillation is a process that synthesizes compact summaries of
large datasets, enabling efficient model training and
inference. Preserving essential information while reducing size allows
quicker processing and improved performance in various applications.
In \citet{sachdeva2023data}, the authors give an extensive survey on
data distillation and propose a bilevel optimization model to handle
the task (see Section~2 in \citet{sachdeva2023data}), which is the same
model described here.
We optimize the lower-level problem exactly, and the upper-level problem with a stochastic projected gradient descent.
We test the \emph{relaxation and rounding},  with both \emph{simple} and \emph{top-$k$ hard thresholding} rounding \citep{zhangadvancing}, and \emph{relax and penalize} strategies over two real datasets, showing the effectiveness of the last one.

\paragraph{Problem setting and formulation.}

Given a training dataset that needs to be distilled
$\mathcal{D}^{\text{train}}=\lbrace(x^{\text{train}}_i,y^{\text{train}}_i)\rbrace _{i =
  1}^{m}$ and a data budget $\tau \in \mathbb{Z}^{+}$, data distillation
techniques aim to synthesize a high-fidelity data summary
$\mathcal{D}^{\text{train}}_{\text{syn}}=\lbrace(x^{\text{train}}_i,y^{\text{train}}_i)\rbrace _{i
  = 1}^{\tau}$ with $\tau\ll m$.
Given a validation set
$\mathcal{D}^{\text{val}}=\lbrace(x^{\text{val}}_j,y^{\text{val}}_j)\rbrace
_{j = 1}^{n}$, the data distillation task can be formulated as the
following bilevel optimization problem
\begin{equation}
  \label{prob:datahypercleaning}
  \min_{v\in\{0,1\}^m, e^\top v=\tau} \ \ell^{\text{val}} (\zeta(v))
  \quad \text{with} \quad
  \zeta(v) = \argmin\limits_{\zeta} \
    \ell^{\text{train}}(\zeta,v) + s \norm{\zeta}^2
\end{equation}
with validation loss~$\ell^{\text{val}}$, weighted training loss
$\ell^{\text{train}} = \sum_{i=1}^m v_i
\ell(x_i^{\text{train}},y_i^{\text{train}},\zeta)$,
regularization parameter~\mbox{$s \in \mathbb{R}$}, $v$ being a binary vector
of dimension~$m$ indicating the samples in $\mathcal{D}^{\text{train}}_{\text{syn}}$,
and $e^\top v=\tau$ is the knapsack constraint to take into
account the budget.
Therefore, we wish to have $v_i = 1$ for the $\tau$ most representative
samples.

A simple approach to solve Problem~\eqref{prob:datahypercleaning}
  is to relax it by allowing the $v_i$'s to vary within the
  interval~$[0, 1]$, projecting them to the region defined by the knapsack
  constraint, and then rounding them at the end.
  In particular, in our case we suppose $x_i,x_j \in \R^{d}$ as well as
  $y_i,y_j \in \R^{e}$ and we consider $\zeta = (W,b)$ a linear model
  with weight $W$ and bias $b$ as well as $\ell^{\text{val}}$ and
  $\ell$ being mean squared error losses.
  Therefore, the aim is to solve the bilevel integer problem
\begin{equation}
  \label{prob:datadistillation}
  \begin{aligned}
    &\min_{v \in [0,1]^{m}, e^\top v=\tau} \ 
    \frac{1}{2n} \sum_{j=1}^{n} \norm{W(v)x^{\text{val}}_j + b(v) -
      y_j^{\text{val}}}^2
    \\&
    \text{with} \quad (W(v),b(v)) = \argmin\limits_{W \in \R^{e \times
        d}, b \in \R^e} \frac{1}{m} \sum_{i=1}^{m} v_i
    \norm{Wx_i^{\text{train}} + b - y_i^{\text{train}}}^2 +
    s\norm{W}^2.
  \end{aligned}
\end{equation}
We indicate with $\Y = \left\{ v \in [0,1]^{m} \colon e^\top v=r\right\}$, and with $\Ybin=\Y \cap \{0,1\}^m$. 

\paragraph{Proposed method.}

Using the approach presented in Section~\ref{sec:penalty-method},
  we avoid the final thresholding and directly optimize the following
  problem
\begin{equation}
  \label{prob:20240607}
  \begin{aligned}
    &\min_{v \in \{0,1\}^{m}, e^\top v=\tau}\  
    \frac{1}{2n} \sum_{j=1}^{n} \norm{W(v)x^{\text{val}}_j + b(v) -
      y_j^{\text{val}}}^2
    \\&
    \text{with} \quad (W(v),b(v)) = \argmin\limits_{W \in \R^{e \times
        d}, b \in \R^e} \frac{1}{m} \sum_{i=1}^{m} v_i
    \norm{Wx_i^{\text{train}} - b - y_i^{\text{train}}}^2 +
    s\norm{W}^2.
  \end{aligned}
\end{equation}
We solve the lower-level problem exactly, obtaining
\begin{equation}
\label{eq:ll}
W(v) = \left( \frac{1}{m} C_v(X,Y) \right) \left( \frac{1}{m} C_v(X) +
  s \mathrm{I}_d\right)^{-1}, \quad b(v) = \bar y_v - W(v)\bar x_v,
\end{equation}
with
\begin{equation*}
  C_v(X,Y) = \sum_{i=1}^{m} v_i (y^{\text{train}}_i-\bar
  y_v)(x^{\text{train}}_i - \bar x_v)^\top
\end{equation*}
being the cross-covariance matrix,
\begin{equation*}
  C_v(X) = \sum_{i=1}^{m} v_i (x^{\text{train}}_i-\bar
  x_v)(x^{\text{train}}_i - \bar x_v)^\top
\end{equation*}
being the covariance matrix, and
\begin{equation*}
  \bar x_v = \frac{1}{\sum_{i=1}^m v_i} \sum_{i=1}^m v_ix_i,
  \quad
  \bar y_v = \frac{1}{\sum_{i=1}^m v_i} \sum_{i=1}^m v_iy_i.
\end{equation*}
All the results are averaged over $5$ runs.
See Section \ref{sec:DHC_auxiliary} for more details.

We can also write the problem as
  \begin{equation}
    \label{eq:20240607}
    \min_{v}  \ell^{\text{val}} (\zeta(W(v),b(v)))
    \quad \text{with}\quad
    v \in \{0,1\}^m, \ e^\top v=\tau.
  \end{equation}
We solve the upper-level problem by means
of a projected stochastic gradient descent method. To efficiently
perform the projection over the knapsack constraint, we use the Kiwiel
algorithm \citep{kiwiel2008breakpoint}; see Section
\ref{sec:DHC_auxiliary} for further details.
Now, since $W(v)$ and $b(v)$ are smooth w.r.t.~$v$,
the objective in \eqref{eq:20240607} satisfies Assumption~\ref{ass:a1}
and hence, in view of Theorem~\ref{teo:equiv} and
Theorem~\ref{teo:limitpointPk}, we can consider the \emph{relaxed and
  penalized} version of Problem~\eqref{eq:20240607}, i.e.,
\begin{equation}
  \label{eq:approximate3}
  \min_{v} \ell^{\text{val}} (\zeta(w(v),b(v))) + \frac 1
  \varepsilon \varphi(v)
  \quad \text{with}\quad
  v \in [0,1]^m, \ e^\top v = \tau
\end{equation}
and state that, if $\varepsilon$ is small enough, the two problems
\eqref{eq:20240607} and \eqref{eq:approximate3} share the same
global minimizers.

\paragraph{Numerical experiments.}
\label{sec:DHC_exp}

\setlength{\tabcolsep}{2.5pt}
\begin{table}[t]
\caption{
Experiments on the regression task across two real-world datasets: \emph{music} and \emph{blog}. The distillation budget $\tau$ is varied at $10\,\%$, $20\,\%$, and $30\,\%$ (\emph{perc} column) of the training set size. For the \emph{relaxation and rounding}, with both \emph{simple} (simple) and  \emph{top-$k$ hard thresholding} rounding (top-k), and the \emph{relax and penalize} (\emph{penalize}) methods, we report (mean $\pm$ standard deviation over 5 runs) the values of the upper level function (before rounding in brackets) for the validation and test sets, the cardinality of the synthesized set, and the RMSE.}
\label{tab:table3}
\centering
\begin{scriptsize}
\begin{NiceTabular}{lccl|cccc}
dataset & \emph{perc}& $\tau$ & method & $\ell^{\text{val}}$ & $\left|D^{\text{train}}_{\text{syn}}\right|$ & $\ell^{\text{test}}$ & RMSE  \\
\midrule
    \emph{music} & 10 & 23186 & simple & 126.85 $\pm$ 1.37 (58.57 $\pm$ 0.02) & 12764.20 $\pm$ 28.63 & 128.61 $\pm$ 1.37 & 16.04 $\pm$ 0.09 \\ 
        ~ & ~ & ~ & top-k & 65.60 $\pm$ 0.32  (58.57 $\pm$ 0.02) & 23186.00 $\pm$ 0.00 & 67.02 $\pm$ 0.30 & 11.58 $\pm$ 0.03 \\ 
    ~ & ~ & ~ & penalize & \pmb{$64.87 \pm 0.70$} & \pmb{$23186.00 \pm 0.00$} & \pmb{$65.96 \pm 0.08$} & \pmb{$11.49 \pm 0.01$} \\ 
        ~ & 20 & 46371 & simple & 84.84 $\pm$ 0.40 (53.37 $\pm$ 0.01) & 24967.00 $\pm$ 47.02 & 86.19 $\pm$ 0.41 & 13.13 $\pm$ 0.03 \\ 
        ~ & ~ & ~ & top-k & 54.61 $\pm$ 0.02 (53.37 $\pm$ 0.01) & 46371.00 $\pm$ 0.00 & 55.82 $\pm$ 0.03 & 10.57 $\pm$ 0.00 \\ 
    ~ & ~ & ~ & penalize & \pmb{$54.20 \pm 0.02$} & \pmb{$46371.00 \pm 0.00$} & \pmb{$55.45 \pm 0.03$} & \pmb{$10.53 \pm 0.00$} \\ 
        ~ & 30 & 69557 & simple & 64.91 $\pm$ 0.11 (50.93 $\pm$ 0.00) & 38250.20 $\pm$ 23.68 & 66.15 $\pm$ 0.11 & 11.50 $\pm$ 0.01 \\ 
        ~ & ~ & ~ & top-k & 50.71 $\pm$ 0.02 (50.93 $\pm$ 0.00)  & 69557.00 $\pm$ 0.00 & 51.91 $\pm$ 0.02 & 10.19 $\pm$ 0.00 \\ 
    ~ & ~ & ~ & penalize & \pmb{$50.52 \pm 0.01$} & \pmb{$69557.00 \pm 0.00$} & \pmb{$51.69 \pm 0.01$} & \pmb{$10.17 \pm 0.00$} \\ \midrule
    \emph{blog} & 10 & 5240 & simple & 2890.77 $\pm$ 240.03 (297.17 $\pm$ 0.10) & 515.00 $\pm$ 5.52 & 3048.34 $\pm$ 276.60 & 78.02 $\pm$ 3.47 \\ 
        ~ & ~ & ~ & top-k & 373.36 $\pm$ 5.20 (297.17 $\pm$ 0.10) & 5240.00 $\pm$ 0.00 & 287.98 $\pm$ 4.36 & 24.00 $\pm$ 0.18 \\ 
    ~ & ~ & ~ & penalize & \pmb{$326.64 \pm 23.61$} & \pmb{$5240.00 \pm 0.00$} & \pmb{$217.87 \pm 1.74$} & \pmb{$20.87 \pm 0.08$} \\ 
        ~ & 20 & 10479 & simple & 2209.31 $\pm$ 35.05 (292.37 $\pm$ 0.01) & 712.80 $\pm$ 4.02 & 2283.30 $\pm$ 41.51 & 67.57 $\pm$ 0.61 \\ 
        ~ & ~ & ~ & top-k & 306.08 $\pm$ 1.76 (292.37 $\pm$ 0.01) & 10479.00 $\pm$ 0.00 & 219.44 $\pm$ 1.76 & 20.95 $\pm$ 0.08 \\ 
    ~ & ~ & ~ & penalize & \pmb{$296.77 \pm 0.34$} & \pmb{$10479.00 \pm 0.00$} & \pmb{$207.90 \pm 0.43$} & \pmb{$20.39 \pm 0.02$} \\ 
        ~ & 30 & 15719 & simple & 13972.08 $\pm$ 11310.87 (327.78 $\pm$ 29.11) & 10447.00 $\pm$ 23265.78 & 15134.38 $\pm$ 10849.34 & 156.81 $\pm$ 84.26 \\ 
        ~ & ~ & ~ & top-k & 360.31 $\pm$ 27.59 (327.78 $\pm$ 29.11) & 15719.00 $\pm$ 0.00 & 297.70 $\pm$ 80.29 & 24.23 $\pm$ 3.24 \\ 
    ~ & ~ & ~ & penalize & \pmb{$352.68 \pm 2.16$} & \pmb{$15719.00 \pm 0.00$} & \pmb{$247.87 \pm 1.11$} & \pmb{$22.27 \pm 0.05$} \\
\bottomrule
\end{NiceTabular}
\end{scriptsize}
\vskip-1em
\end{table}

We tested the proposed method on the data distillation problem
  by performing experiments on two real-world datasets.
  First, \emph{music} \citep{year_prediction_msd_203}, a dataset with song
features from 1922 to 2011 used to predict the release year based on
90 attributes, including timbre averages and covariances.
Second, \emph{blog} \citep{blogfeedback_304}, a dataset containing features from blog posts, focused on predicting the number of comments received in the next 24~hours using various attributes.
For each dataset, we address Problem~\eqref{prob:datadistillation} by
using a training set for the lower level and a validation set for the
upper level.
We vary the distillation budget $\tau$ at $10\,\%$, $20\,\%$, $30\,\%$ of
the training set size; see Section~\ref{sec:DHC_auxiliary} for more
details.
We run the experiments with the \emph{relaxation and rounding}, with both \emph{simple} and \emph{top-$k$ hard thresholding} rounding, and the
\emph{relax and penalize} strategies, and we report the results in
Table \ref{tab:table3}. First, we compare the final objective values
of the upper-level problem obtained from \emph{relaxation and
  rounding}, with both rounding strategies, alongside those from \emph{relax and penalize}, including the actual number of training
samples selected in $D^{\text{train}}_{\text{syn}}$.
Next, we evaluate both methods by optimizing the same upper level
function on a test set using the selected $D^{\text{train}}_{\text{syn}}$
sets, reporting the objective value and the root mean square error (RMSE), a common metric for assessing regression model performance.
Similar to the group lasso structures discussed in Section
\ref{sec:GLS}, the \emph{relaxation and rounding},  with both \emph{simple} and \emph{top-$k$ hard thresholding} rounding methodologies, initially achieves a low objective
function value at the upper level, but this value increases after rounding. In contrast, the \emph{relax
  and penalize} method successfully finds an integer solution with a
low objective function value. Additionally, the penalty method  selects a number of training samples close to the budget, while the \emph{relaxation and rounding}     strategy results in a relaxed solution with only a few components above 0.5. The \emph{simple} rounding hence leads to fewer samples and ultimately not utilizing the full
budget, while the rounding based on \emph{top-$k$ hard thresholding} naturally leads to a solution that satisfies the budget constraint. When applying the selected $D^{\text{train}}_{\text{syn}}$
to a test set, the \emph{relax and penalize} method always achieves lower
values for the objective function and RMSE, leading to a better
overall solution with respect to \emph{relaxation and rounding}, both \emph{simple} and \emph{top-k}.
In Section \ref{sec:DHC_auxiliary4}, we report the evolution along the different steps of the quantities reported in Table \ref{tab:table3}  for one setting.

\section{Conclusion}
\label{sec:conc}

In this paper, we studied the idea of relaxing the integrality
constraints and using a penalty term to handle mixed-binary bilevel
optimization problems arising in hyperparameter tuning of machine
learning systems. Besides a result
concerning the equivalence in terms of global minimizers, sufficient
conditions for identifying mixed-binary local minimizers are
stated. These theoretical results naturally lead to devise a penalty
method that is, under suitable assumptions, guaranteed to provide
mixed-binary solutions.
The reported numerical results highlight that our method is competitive with state-of-the-art approaches based on relaxation and rounding for the group lasso problem and outperforms these methods on the data distillation task.

\bibliography{ref}
\bibliographystyle{tmlr}

\appendix
\section{Auxiliary results}
\label{sec:auxiliary}

In this section, we give a number of technical results related to the
penalty function $\varphi$ used throughout the paper. We recall that
$\varphi\colon[0,1]^{\ny}\to \R$ is defined as
\begin{equation*}
  \varphi(\y) = \sum_{i=1}^{\ny} \y_i(1 - \y_i).
\end{equation*}
Moreover, we denote by $[\ny]$ the set $\{1,\dots, \ny\}$ and by
$\norm{\cdot}$ the Euclidean norm in $\R^{\ny}$. Occasionally, we will
also used the norms
\begin{equation*}
\norm{\y}_1 = \sum_{i=1}^{\ny} \abs{\y_i}
\quad\text{and}\quad
\norm{\y}_\infty = \max_{1 \leq i \leq \ny}\abs{\y_i}.
\end{equation*}

\begin{lemma}
  \label{lem1}
  Let $\psi\colon [0,1]\to \R$ be such that
  \begin{equation*}
    \forall\, t \in [0,1]\colon\quad \psi(t) = t(1-t).
  \end{equation*}
  Let $\sigma \in \left]0,1/2\right]$.
  Then, for every $t_1, t_2 \in [0,1]$ we have
  \begin{equation*}
    \bigg\lvert \frac{t_1+t_2}{2} - \frac{1}{2} \bigg\rvert \geq
    \sigma \implies \abs{\psi(t_2)-\psi(t_1)} \geq 2\sigma
    \abs{t_2 - t_1}.
  \end{equation*}
\end{lemma}
\begin{proof}
  Let $t_1, t_2 \in [0,1]$. One can easily check that
  \begin{equation*}
    \psi(t_2)-\psi(t_1) = (1- (t_1+t_2))(t_2 - t_1).
  \end{equation*}
  Therefore,
  \begin{equation*}
    \abs{\psi(t_2)-\psi(t_1)} = \abs{t_1+t_2 - 1}\abs{t_2 - t_1}
    = 2 \bigg\lvert \frac{t_1+t_2}{2} - \frac{1}{2}\bigg\rvert \abs{t_2 - t_1}
    \geq 2 \sigma \abs{t_2 - t_1}.
    \qedhere
  \end{equation*}
\end{proof}

\begin{lemma}
  \label{lem2}
  Let $\sigma \in \left]0,1/2\right]$ and $\y, \y^\prime \in
  [0,1]^{\ny}$ be such that the following holds:
  \begin{equation*}
    \forall\, i \in [\ny]\colon\ \y^\prime_i \neq \y_i\
    \implies
    \bigg\lvert \frac{\y_i+\y_i^\prime}{2} - \frac 1 2 \bigg\rvert \geq
    \sigma
    \quad\text{and}\quad
    \bigg\lvert \y_i^\prime - \frac 1 2\bigg\rvert \leq \bigg\lvert \y_i
    - \frac 1 2\bigg\rvert.
  \end{equation*}
  Then, $\varphi(\y^\prime)-\varphi(\y) \geq 2 \sigma \norm{\y^\prime -
    \y}$.
\end{lemma}
\begin{proof}
  Let $\y,\y^\prime \in [0,1]^{\ny}$ be as in the statement and let $I =
  \{i \in [\ny] \,\vert\, \y_i \neq \y^\prime_i\}$.
  Then,
  \begin{equation*}
    \forall\, i \in I\colon\ \
    \bigg\lvert \frac{\y_i+\y_i^\prime}{2} - \frac 1 2 \bigg\rvert \geq \sigma
    \quad\text{and}\quad \psi(\y_i) \leq \psi(\y^\prime_i)
  \end{equation*}
  and hence, by Lemma~\ref{lem1}, we have
  \begin{align*}
    \varphi(\y^\prime) - \varphi(\y)
    &= \sum_{i \in I} \psi(\y^\prime_i) - \psi(\y_i)
    \\
    &= \sum_{i \in I} \abs{\psi(\y^\prime_i) - \psi(\y_i)} \geq 2 \sigma
      \sum_{i \in I} \abs{\y^\prime_i - \y_i}
    \\
    & = 2 \sigma \norm{\y^\prime - \y}_1 \geq 2\sigma \norm{\y^\prime - \y},
  \end{align*}
  where we used $\norm{\cdot} \leq \norm{\cdot}_1$ for the last
  inequality.
\end{proof}

\begin{remark}
  Lemma~\ref{lem2} says that if the components of the mid point
  ($\y+\y^\prime)/2$ are bounded away from~$1/2$ and the componentwise
  distance from $\y$ to $1/2$ is larger than that of $\y^\prime$ to $1/2$,
  then $\varphi(\y^\prime)- \varphi(\y)$ can be bounded
  from below by $\norm{\y^\prime-\y}$; up to a multiplicative
  constant.
\end{remark}

\begin{remark}
  \label{rmk1}
  The conditions on $\y$ and $\y^\prime$ required by Lemma~\ref{lem2}
  are satisfied if
  \begin{equation*}
    \forall\, i \in [\ny]\colon\ \y_i\neq \y^\prime_i\ \implies \
    \begin{cases}
      (\y_i - 1/2)(\y^\prime_i - 1/2)\geq 0,\\
      \abs{\y^\prime_i - 1/2}\leq\abs{\y_i - 1/2},\\
      2\sigma \leq \abs{\y_i - 1/2}.
    \end{cases}
  \end{equation*}
  Indeed, in such case we have
  \begin{align*}
    \bigg\lvert \frac{\y_i+\y_i^\prime}{2} - \frac 1 2 \bigg\rvert
    & = \frac{1}{2} \abs{\y_i + \y_i^\prime - 1}
    \\
    & = \frac 1 2 \bigg\lvert \y_i - \frac 1 2 + \y_i^\prime - \frac1 2
      \bigg\rvert \overset{(*)}{=} \frac 1 2 \bigg(\bigg\lvert \y_i -
      \frac 1 2  \bigg\rvert + \bigg\lvert \y^\prime_i - \frac 1 2
      \bigg\rvert \bigg)
      \geq \frac 1 2 \bigg\lvert \y_i - \frac 1 2  \bigg\rvert \geq \sigma,
  \end{align*}
  where the equality in $(*)$ is due to the fact that
  \begin{equation*}
    (\y_i - 1/2)(\y^\prime_i - 1/2) \geq 0, \
    \abs{\y^\prime_i - 1/2} \leq \abs{\y_i - 1/2}
    \implies
    \big( \y_i \leq \y^\prime_i \leq 1/2\ \ \text{or}\ \ \y_i \geq
    \y_i^\prime \geq 1/2 \big).
  \end{equation*}
\end{remark}

\begin{corollary}
\label{cor2}
Let $\sigma \in \left]0,1/2\right]$ and $\y \in \{0,1\}^{\ny}$.
Then
\begin{equation*}
  \forall\, \y^\prime \in [0,1]^{\ny}\colon\ \ \norm{\y^\prime - \y}
  \leq 1 - 2\sigma\ \implies \ \varphi(\y^\prime) \geq 2\sigma
  \norm{\y^\prime - \y}.
\end{equation*}
\end{corollary}
\begin{proof}
  Let $\y \in \{0,1\}^{\ny}$ and $\y^\prime \in [0,1]^{\ny}$.
  We will check the conditions in Lemma~\ref{lem2}.
  Let $i \in [\ny]$. Since $\abs{\y_i - 1/2} = 1/2$, the condition
  $\abs{\y^\prime - 1/2} \leq \abs{\y_i- 1/2}$ is automatically satisfied.
  Now, we note that
  \begin{align*}
    \y_i
    &=0 \implies \bigg\lvert \frac{\y_i+\y_i^\prime}{2} - \frac 1 2
      \bigg\rvert = \bigg\lvert \frac{\y_i^\prime}{2} - \frac 1 2
      \bigg\rvert = \frac 1 2 \abs{\y^\prime_i - 1} = \frac1 2 (1 -
      \abs{\y^\prime_i - \y_i}),
    \\
    \y_i
    & = 1 \implies \bigg\lvert \frac{\y_i+\y_i^\prime}{2} - \frac 1
      2 \bigg\rvert = \bigg\lvert \frac{\y_i^\prime}{2} \bigg\rvert =
      \frac 1 2 \abs{\y^\prime_i} = \frac1 2 (1 - \abs{\y^\prime_i -
      \y_i}).
  \end{align*}
  Therefore,
  \begin{equation*}
    \bigg\lvert \frac{\y_i+\y_i^\prime}{2} - \frac 1 2 \bigg\rvert \geq \sigma
    \iff
    1 - \abs{\y^\prime_i - \y_i} \geq 2\sigma
    \iff
    \abs{\y^\prime_i - \y_i} \leq 1 - 2\sigma
  \end{equation*}
  and the first condition in Lemma~\ref{lem2} is then equivalent to
  the condition $\norm{\y^\prime - \y}_{\infty} \leq 1 - 2\sigma$.
  The statement follows by recalling that $\norm{\cdot}_\infty \leq
  \norm{\cdot}$.
\end{proof}

\begin{remark}
It is clear from the proof of Corollary~\ref{cor2} that, in fact, it holds
\begin{equation*}
  \forall\, \y^\prime \in [0,1]^{\ny}\colon\ \ \norm{\y^\prime - \y}_{\infty}
  \leq 1 - 2\sigma\ \implies \ \varphi(\y^\prime) \geq 2\sigma
  \norm{\y^\prime - \y}.
\end{equation*}
\end{remark}

\begin{lemma}
  \label{lemma:varphi}
  Let~$\y\in \{0,1\}^{\ny}$, $\bar \y\in [0,1]^{\ny}$, and $c \in \R$ be such that
  $\norm{\y-\bar \y}_{\infty} < c <\frac 1 2$.
  Let \[\y^t = (1-t)\bar \y + t\y \quad \text{with} \quad t\in [0,1].\]
  Then,
\begin{equation*}
\varphi(\bar \y) - \varphi(\y^t) \geq (1-2c) \norm{ \y^t-\bar \y}.
\end{equation*}
\end{lemma}
\begin{proof}
  Since $\lvert \y_i - \bar \y_i\rvert < c < \frac 1 2$ holds for all
  $i=1,\ldots, n$, we have
  \[
    \frac 1 2 = \bigg\lvert \y_i - \frac 1 2\bigg\rvert \leq
    \lvert \y_i - \bar \y_i\rvert + \bigg\lvert \bar \y_i - \frac 1
    2\bigg\rvert < c + \bigg\lvert \bar \y_i - \frac 1 2\bigg\rvert,
  \]
  which implies
  \[\bigg\lvert \bar \y_i - \frac 1 2\bigg\rvert > \frac 1 2 - c.\]
  Moreover, since $\y_i\in \{0,1\}$ and $\lvert \bar \y_i - \y_i\rvert <
  c < \frac 1 2$, we have
  \begin{align*}
    \y_i = 0 & \implies \y_i = 0 \leq \bar \y_i < c < \frac 1 2, \\
    \y_i = 1 & \implies \frac 1 2 < 1-c < \bar \y_i \leq 1 = \y_i,
  \end{align*}
  and, hence, since $\y_i^t$ is between $\y_i$ and $\bar \y_i$, it holds
  \[
    \bigg\lvert \y_i^t - \frac 1 2\bigg\rvert \geq \bigg\lvert \bar \y_i - \frac 1 2\bigg\rvert > \frac 1 2 - c
  \]
  so that
  \begin{align*}
    \bigg\lvert \frac{\y_i^t + \bar \y_i}{2} - \frac 1 2\bigg\rvert &= \bigg\lvert \frac 1 2 (\y_i^t + \bar \y_i -1)\bigg\rvert = \frac 1 2 \bigg\lvert \bigg(\y_i^t - \frac 1 2\bigg) + \bigg(\bar \y_i - \frac 1 2\bigg)\bigg\rvert \\
    &= \frac 1 2 \bigg(\bigg\lvert \y_i^t - \frac 1 2 \bigg\rvert + \bigg\lvert \bar \y_i - \frac 1 2\bigg\rvert\bigg) \geq
    \bigg\lvert \bar \y_i - \frac 1 2\bigg\rvert > \frac 1 2 -c.
  \end{align*}

  Therefore, by Lemma~\ref{lem2},
  $\varphi(\bar \y) - \varphi(\y^t) \geq (1-2c) \| \y^t-\bar \y\|$
  holds.
\end{proof}

\section{Proof of Theorem~\ref{teo:equiv}}

For the sake of brevity, we set
\begin{equation*}
  S = \argmin\limits_{(\y,\x) \in \W} G(\y,\x)
  \quad\text{and}\quad
S(\varepsilon) = \argmin_{(\y,\x) \in \X \times \Y}G(\y,\x)+\frac{1}{\varepsilon} \varphi(\y).
\end{equation*}
Recall that $\Ybin = \Y\cap \{0,1\}^{\ny}$.
  Let $\rho \in \left]0,1\right[$ and let $\hat{\varepsilon} \in \left]0,
    (1-\rho)/L\right[$.
  We define the open set
  \begin{equation*}
    U = \bigcup_{\y \in \Ybin} B_\rho(\y),
  \end{equation*}
where $\rho$ is chosen small enough so to ensure
that $\Y\setminus U \neq \varnothing$.\footnote{This means that
  there exists a $\y^* \in \Y$ such that for every $\y \in \Ybin$, it holds
  $\norm{\y - \y^*}>\rho$.
  This condition is met if we pick $\y^* \in \Y\setminus\{0,1\}^{\ny}$
  (see Assumption \ref{ass:a1}\ref{ass:theta}) and (taking into
  account that $\Ybin=\Y\cap\{0,1\}^{\ny}$ is a finite set) choose
  $\rho$ such that $0<\rho<\inf_{\y \in \Ybin} \norm{\y - \y^*}$.}
  Let $\bar{\y}$ be a minimizer of $\varphi$ over the compact
  set~$\Y\setminus U$.
  Then, clearly
  \begin{equation}
        \label{eq:20230427a}
    \forall\, \y^\prime \in \Y\setminus U\colon \quad
    \varphi(\y^\prime) \geq \varphi(\bar{\y})>0.
  \end{equation}
  (Note that, since $\bar{\y} \notin U$, then $\bar{\y} \notin
  \{0,1\}^{\ny}$, and hence there exists $i \in [\ny]$ such that
  $\bar{\y}_i \in \left]0,1\right[$, which implies that
  $\varphi(\bar{\y}) \geq \bar{\y}_i (1 - \bar{\y}_i)>0$.)
  Thus, since
  \begin{equation*}
    \lim_{\varepsilon\to 0^+} \frac{1}{\varepsilon} \varphi(\bar{\y}) =
    +\infty,
  \end{equation*}
  there exists $\tilde{\varepsilon} \in
  \left]0,\hat{\varepsilon}\right]$ such that
  \begin{equation}
    \label{eq:20230218b}
    \forall\, \varepsilon \in \left]0,\tilde{\varepsilon}\right]\colon
    \quad
    \frac{1}{\varepsilon} \varphi(\bar{\y})
    > \sup_{\X \times\Ybin} G - \inf_{\X\times(\Y\setminus U)} G.
  \end{equation}
  Now, we let $\varepsilon \in \left]0,\tilde{\varepsilon}\right]$ and
  show that $S(\varepsilon) \subset \Ybin$.
  Let $(\x^*, \y^*) \in S(\varepsilon)$ and suppose, by
  contradiction, that $(\x^*, \y^*) \notin \Ybin$.
  We can have the following two cases:
  \begin{enumerate}
  \item[(a)] Let $\y^* \in U$. Then, there exists $\y \in
    \Ybin$ such that $\y^* \in \Y\cap B_\rho(\y)$. Thus, in view
    of Assumption~\ref{ass:a1} and Corollary~\ref{cor2}, we have
    \begin{align*}
       G(\x^*,\y)-G(\x^*, \y^*)
      \leq \
L \norm{\y^* - \y} < \frac{1-\rho}{\varepsilon}\norm{\y^* - \y}
        \leq  \frac{1}{\varepsilon} \varphi(\y^*)
        =
        \frac{1}{\varepsilon} \varphi(\y^*)
        - \frac{1}{\varepsilon} \varphi(\y)
    \end{align*}
    and, hence,
    \begin{equation*}
      G(\x^*,\y) + \frac{1}{\varepsilon} \varphi(\y) < G(\x^*,
      \y^*)+\frac{1}{\varepsilon} \varphi(\y^*),
    \end{equation*}
    which yields a contradiction since $(\x^*, \y^*)$ is a global
    minimizer of Problem~\eqref{prob:ref}.
  \item[(b)] Let $\y^* \notin U$.
    Then $(\x^*, \y^*) \in \X\times(\Y\setminus U)$ and hence, recalling \eqref{eq:20230427a} and \eqref{eq:20230218b},
    we have
    \begin{align*}
      G(\x^*, \y^*) + \frac{1}{\varepsilon} \varphi(\y^*)
      &\geq \inf_{\X\times(\Y\setminus U)} G +
        \frac{1}{\varepsilon} \varphi(\y^*)
      \\
      & \geq \inf_{\X\times(\Y\setminus U)}
        G + \frac{1}{\varepsilon} \varphi(\bar{\y})
      \\
      & > \sup_{\X\times\Ybin} G \\
      &\geq G(\y,\x) + \frac{1}{\varepsilon}
        \underbrace{\varphi(\y)}_{=0}
    \end{align*}
    for any $(\y,\x)\in \X\times\Ybin \subseteq \X\times\Y$, which gives again a contradiction.
  \end{enumerate}
  Thus, in both cases we get a contradiction and therefore necessarily
  \mbox{$(\x^*, \y^*) \in \W$}.
  Now, if we take $(\x^*,\y^*) \in S(\varepsilon)$, since
  $(\x^*,\y^*) \in \W$, we have
  \begin{equation*}
    G(\x^*,\y^*) + \frac{1}{\varepsilon} \underbrace{\varphi(\y^*)}_{=0} \leq
    G(\y,\x) +
    \frac{1}{\varepsilon}\underbrace{\varphi(\y)}_{=0}
    \quad\forall\,(\y,\x)\in \W\subseteq \X\times\Y.
  \end{equation*}
  Thus, for all $(\y,\x) \in \W$, $G(\x^*,\y^*) \leq G(\y,\x)$, meaning $(\x^*,\y^*) \in
  S$.
  Vice versa, let $(\x^*,\y^*) \in S$. Choosing $(\tilde \x^*,\tilde \y^*) \in
  S(\varepsilon)$, since $(\tilde \x^*,\tilde \y^*) \in \W$,
  we have
  \begin{align*}
    G(\x^*,\y^*) + \frac{1}{\varepsilon} \underbrace{\varphi(\y^*)}_{=0} & =
    G(\x^*,\y^*) \leq G(\tilde \x^*,\tilde \y^*)\\
    &\leq  G(\tilde \x^*,\tilde \y^*) + \frac{1}{\varepsilon} \varphi(\tilde{\y}^*)
    = \min_{(\y,\x)) \in \X\times\Y} G(\y,\x) + \frac{1}{\varepsilon}\varphi(\y).
  \end{align*}
  Hence, $(\x^*,\y^*) \in S(\varepsilon)$.

\section{Details on the group lasso application}
\label{sec:GLS_auxiliary}

In this section, we report on the details regarding the
  group-sparsity structure estimation in regression problems presented
  in Section~\ref{sec:GLS}.
Firstly, we discuss the extension of the algorithm presented in~\citet{2018_Frecon_J_p-neurips_biglasso} to the mix-integer
case. Secondly, we report the details of the performed experiments.

\subsection{Extensions to \citep{2018_Frecon_J_p-neurips_biglasso}}

In~\citet{2018_Frecon_J_p-neurips_biglasso},
  Problem~\eqref{eq:20230503a} is considered only in the integer hyperparameter $\theta$, while the real hyperparameter $\lambda$ is supposed to be fixed.
  Here, we report the extension to the
  optimization in both the hyperparameters $\theta$ and $\lambda$.
  In particular,~\citet{2018_Frecon_J_p-neurips_biglasso} do not solve
  the lower-level problem exactly, rather consider the following
  approximate problem, providing conditions under which it converges to
  the exact one as the number of inner iterations $q$ grows:
\begin{equation*}
\min_{(\y,\x) \in \X\times \Y} \mathcal{U}^{(q)} (\y,\x)
\quad \text{with} \quad
\begin{cases}
  u^{(0)} \equiv 0 \in \mathbb{R}^{P \times L}, \\
  \forall i=0,1,\ldots,q-1: \ u^{(i+1)}(\y,\x) = \mathcal{A}(u^{(i)}(\y,\x),\y,\x), \\
  w^{(q)} (\y,\x) = \mathcal{B}(u^{(q)}(\y,\x),\y,\x), \\
  \mathcal{U}^{(q)} (\y,\x) = \frac 1 T \sum_{t=1}^T C_t(w^{(q)}(\y,\x))
\end{cases}
\end{equation*}
and  $\mathcal{A}: \mathbb{R}^{P \times L} \times \X \times \Y
\rightarrow \mathbb{R}^{P \times L}$ as well as $\mathcal{B}:
\mathbb{R}^{P \times L} \times \X \times \Y \rightarrow
\mathbb{R}^{P}$. We denote by $\partial_1 \mathcal{A}(u,\y,\x)$ the
partial derivatives of $\mathcal{A}$ with respect the variable $u$ and
$\partial_2 \mathcal{A}(u,\y,\x)$ the partial derivatives of
$\mathcal{A}$ with respect the variables $\x$ and $\y$.
The same
notation is used for the partial derivatives of $\mathcal{B}$.
When specializing to the case of group lasso, $\mathcal{A}$ and
$\mathcal{B}$ take the expression reported in Section~B.1 of~\citet{2018_Frecon_J_p-neurips_biglasso} supplementary material:
\begin{align*}
  \mathcal{A}(u,\y,\x)
  & = \nabla\Phi_{\x}^* (\nabla\Phi_{\x}(u) + \gamma
    A_{\theta}\mathcal{B}(u,\y,\x)),
  \\
  \mathcal{B}(u,\y)
  & = \nabla f^* (-A_{\theta}^\top u)
\end{align*}
with $\Phi_{\x}^*$
being the separable Hellinger-like function as defined in
Definition~3.2 in~\citet{2018_Frecon_J_p-neurips_biglasso}, $\gamma
>0$ is some given step-size, $f^*$ is the Fenchel conjugate of $f$, and $A_{\theta}^\top$
is the transpose of the operator~$A_{\theta}$ as defined in
Problem~3.1 and Problem~3.2 in~\citet{2018_Frecon_J_p-neurips_biglasso}.
Therefore, noticing that the dependence on~$\y$ is hidden in
$\Phi_{\x}$, we only need to update $\partial_2 \mathcal{A}(u,\y,\x)$,
because $\mathcal{B}$ does not depend on $\y$ and $\partial_1
\mathcal{A}(u,\y,\x)$ is the derivative by $u$.
We recall that, for every $u = (u_l)_{1\geq l \leq L} \in
\mathbb{R}^{P \times L}$ and $v = (v_l)_{1\geq l \leq L} \in
\mathbb{R}^{P \times L}$ for every $l=1,\ldots,L$,
\begin{equation*}
  \nabla_l\Phi(u) = \nabla\phi(u_l) =
  \frac{u_l}{\sqrt{\x^2-\norm{u_l}_2^2}}
  \quad \text{and} \quad
  \nabla_l\Phi^*(v) = \nabla\phi^*(v_l) =
  \frac{v_l}{\sqrt{\x^2-\norm{v_l}_2^2}}
\end{equation*}
holds.
Therefore, we obtain
\begin{equation*}
  \mathcal{A}^{(l)}(u,\y,\x)
  = \x \left( \frac{v_l(\x)}{\sqrt{1+\norm{v_l(\x)}_2^2}} \right)
\end{equation*}
and
\begin{equation*}
 \partial_{\x}\mathcal{A}^{(l)}(u,\y,\x) = \frac{v_l(\x) + \x
   v_l'(\x)}{\sqrt{1+\norm{v_l(\x)}_2^2}} - \frac{\x v_l(\x)\langle
   v_l(\x),v_l'(\x)\rangle}{(1+\norm{v_l(\x)}_2^2)^{\frac{3}{2}}}
\end{equation*}
with
\begin{equation*}
  v_l(\x) = \frac{1}{\sqrt{1+\norm{u_l}_2^2}} \cdot u_l + \gamma \theta_l \odot \mathcal{B}(u,\y) \quad \text{and} \quad
  v'_l(\x) = - \frac{\x}{(1+\norm{u_l}_2^2)^{\frac{3}{2}}} \cdot u_l.
\end{equation*}
For clarity, we re-write $v_l(\x)$ and $v'_l(\x)$ as
\begin{equation*}
  v_l(\x) =  \iota \cdot u_l + d_l,
  \quad
  v'_l(\x) = \kappa \cdot u_l
\end{equation*}
with
\begin{equation*}
  \iota = \frac{1}{\sqrt{1+\norm{u_l}_2^2}},
  \quad \kappa =  \frac{-\x}{(1+\norm{u_l}_2^2)^{\frac{3}{2}}},
  \quad d_l = \gamma \theta_l \odot \mathcal{B}(u,\y).
\end{equation*}
Our aim is to compute $\partial_2\mathcal{A}(u,\y,\x)^\top a$ with $a
\in \mathbb{R}^{P\times L}$, where the partial derivative is with
respect to all the hyperparameters, meaning the group of variables
$(\theta,\lambda)$. Considering that
\begin{equation*}
\forall\,(b,\beta)\in \mathbb{R}^{P\times L}\times\R\colon \quad
\partial_2\mathcal{A}(u,\y,\x)(b,\beta) =
\partial_{\y}\mathcal{A}(u,\y,\x) b +
\partial_{\x}\mathcal{A}(u,\y,\x) \beta,
\end{equation*}
we have
\begin{align}
\begin{split}
\label{eq:2}
\langle \partial_2\mathcal{A}(u,\y,\x)(b,\beta),a \rangle &=
 \langle \partial_{\y}\mathcal{A}(u,\y,\x) b,a \rangle +
\beta \langle \partial_{\x}\mathcal{A}(u,\y,\x),a \rangle
\\ &=
 \langle  b,\partial_{\y}\mathcal{A}(u,\y,\x)^\top a \rangle +
 \langle \partial_{\x}\mathcal{A}(u,\y,\x),a \rangle \cdot  \beta
 \\ &=
  \langle  (b,\beta), \left(
  \partial_{\y}\mathcal{A}(u,\y,\x)^\top a, \langle\partial_{\x}\mathcal{A}(u,\y,\x),a\rangle \right) \rangle.
\end{split}
\end{align}
It follows that
\begin{equation*}
\partial_2\mathcal{A}(u,\y,\x)^\top a =
\left(
\partial_{\y}\mathcal{A}(u,\y,\x)^\top a ,
\langle\partial_{\x}\mathcal{A}(u,\y,\x),a\rangle
\right).
\end{equation*}
Thanks to this result, we can use Algorithm 2 in~\citet{2018_Frecon_J_p-neurips_biglasso} extended to the optimization
in both hyperparameters to the hypergradient computation, substituting the expression of $\partial_2 \mathcal{A}(u^{(i)}(\y,\x),\y,\x)^\top
a^{(i+1)}$ in the calculation of $b^{(i)}$ (with $i$ iteration index); see Algorithm
\ref{alg:hypergradient}.
In particular, in our case, we have an additional component of
$b^{(i)}$ regarding the hyperparameter $\x$, that we call $\beta^{(i)}
\in \mathbb{R}$.
Therefore,
\begin{align*}
  \beta^{(i)}
  &= \langle \partial_{\x} \mathcal{A}(u^{(i)}(\y,\x),\y,\x),
  a^{(i+1)} \rangle + \beta^{(i+1)} \\
  &=
  \sum_{l = 1}^{L} \langle \partial_{\x}
  \mathcal{A}^{(l)}(u^{(i)}(\y,\x),\y,\x) , a^{(i+1)}_l \rangle +
  \beta^{(i+1)}
\end{align*}
holds with
\begin{equation*}
\langle \partial_{\x} \mathcal{A}^{(l)}(u^{(i)}(\y,\x),\y,\x) , a^{(i+1)}_l \rangle =
\frac{\langle v_l(\x) + \x v'_l(\x) , a^{(i+1)}_l \rangle}{\sqrt{1+\norm{v_l(\x)}^2}} -
\frac{\x \langle v_l(\x),v'_l(\x) \rangle \langle v_l(\x),a_l^{(i+1)} \rangle}{(1+\norm{v_l(\x)})^{\frac{3}{2}}}.
\end{equation*}

\begin{algorithm2e}[t]
  \DontPrintSemicolon
  \textbf{Require:} Group structure $\y$, number of inner iterations $q$.\\
  \hspace{0.3cm}Initialize $u^{(0)}(\y,\x) \equiv 0 \in \mathbb{R}^{P\times L}$\\
  \hspace{0.3cm}\textbf{for} $i=1$ to $q$ \textbf{do}\\
  \hspace{0.6cm} $u^{(i)}(\y,\x) = \mathcal{A}(u^{(i-1)}(\y,\x),\y,\x)$.\\
  \hspace{0.3cm} \textbf{end for}\\
  \textbf{Output:} 1. $u^{(0)}(\y,\x),\ldots,u^{(q)}(\y,\x)$, $w^{(q)}(\y,\x) = \mathcal{B}(u^{(q)}(\y,\x),\y)$.\\
  \hspace{0.3cm} Initialize $a_q = \partial_1\mathcal{B}(u^{(q)}(\y,\x),\y,\x)^\top \nabla C(x^{(q)}(\y,\x)$,$b_q = \partial_2\mathcal{B}(u^{(q)}(\y,\x),\y,\x)^\top \nabla C(x^{(q)}(\y,\x)$.\\
  \hspace{0.3cm}\textbf{for} $i=q-1$ to $0$ \textbf{do}\\
  \hspace{0.6cm} $a^{(i)} = \partial_1 \mathcal{A}(u^{(i)}(\y,\x),\y,\x)^\top a^{(i+1)}$ \\
  \hspace{0.6cm} $b^{(i)} = \partial_{\y} \mathcal{A}(u^{(i)}(\y,\x),\y,\x)^\top a^{(i+1)} + b^{(i+1)}$ \\
  \hspace{0.6cm} $\beta^{(i)} = \langle \partial_{\x} \mathcal{A}(u^{(i)}(\y,\x),\y,\x), a^{(i+1)} \rangle + \beta^{(i+1)} $ \\
  \hspace{0.3cm}\textbf{end for}\\
  \textbf{Output:} 2. Hypergradients
  $\nabla_{\y} \mathcal{U}^{(q)}(\y,\x) = b^{(0)}$,
  $\nabla_{\x} \mathcal{U}^{(q)}(\y,\x) =\beta^{(0)}$.\\
  \caption{Hypergradient computation (reverse mode)}
  \label{alg:hypergradient}
\end{algorithm2e}

\subsection{Experimental setup}
\label{sec:GLS_setup}


\setlength{\tabcolsep}{2.5pt}
\begin{table}[t]
\caption{Ablation study hyperparameter $\beta$ for the setting inequal group size and $a = 0.3$.}
\label{tab:table11}
\centering
\begin{scriptsize}
\begin{NiceTabular}{l ccc ccc}
$\beta$ & outer iter & $G(\lambda^{p},\theta^{p})$ & $\norm{w(\lambda^{p},\theta^{p}) - w^\star}_F$\\
\midrule
        0.1 & 4.33 $\pm$ 0.58 & 0.06 $\pm$ 0.00 & 5.86 $\pm$ 0.17 \\ 
        0.3 & 7.00 $\pm$ 0.00 & 0.05 $\pm$ 0.00 & 5.71 $\pm$ 0.14 \\ 
        0.5 & $\pmb{11.33 \pm 0.58}$ & $\pmb{0.05 \pm 0.00}$ & $\pmb{5.54 \pm 0.11}$ \\ 
        0.7 & 18.00 $\pm$ 0.00 & 0.06 $\pm$ 0.03 & 5.73 $\pm$ 0.68 \\ 
        0.9 & 47.00 $\pm$ 5.29 & 0.07 $\pm$ 0.03 & 5.87 $\pm$ 1.02 \\  
\bottomrule
\end{NiceTabular}
\end{scriptsize}
\vskip-1em
\end{table}

\begin{figure*}[htb]
  \centering
  \input{data/distinf}

\begin{tikzpicture}
\begin{groupplot}[
        group style={
            group name=dataset,
            columns=1,
            rows=1,
		},
        width = 5.5cm, 
        height = 5cm,
        axis x line*=bottom,
        xtick align=center,
        ytick align=center,
        ymajorgrids,
		grid style={draw=gray!50},
        y axis line style={draw=none},
		ytick style={draw=none},
        yticklabel style={xshift=-1em},
        every node near coord/.style={font=\scriptsize, yshift=-1.6em},
		y tick label style={
        /pgf/number format/.cd,
            fixed,
            precision=3,
        /tikz/.cd
    	},
        xticklabel style={rotate=35, anchor=east, yshift=-5pt, xshift=5pt, font=\footnotesize\sffamily},
        yticklabel style={font=\footnotesize\sffamily, xshift=10pt},
        ylabel style={font=\footnotesize\sffamily, align=center},
        xmin=0,xmax=30000,
        ymax=1,ymin=-0.1,
        xtick distance=1000, 
        xticklabels={0,,,,,0.5,,,,,1.0,,,,,1.5,,,,,2.0,,,,,2.5,,,,,3.0},
        x tick scale label style={yshift=10pt}
        ]
\nextgroupplot[]
    \addplot [thick, color=gray] table[x=x,y=y]{\GroupLasso};
\end{groupplot}
\end{tikzpicture}
  \caption{Evolution of $\dist_\infty (\theta^k, \Ybin)$ for the setting with inequal group sizes and $a = 0.3$.} 
  \label{fig:12}
\end{figure*}

In Section \ref{sec:GLS_exp}, we conduct experiments on synthetic datasets for the group lasso problem. 
Regarding the details of the experiments, we implement the framework of Algorithm~\ref{alg:penalty-algorithm} by solving a sequence of outer optimization problems $P^k$ of type \eqref{prob:approximate2} w.r.t.~$(\x,\y)$. We select $\varepsilon^0=10^5$, such that the penalty term does not dominate the function, making the evaluation at the first point behave similarly to the unconstrained version. In order to better exploit the hyperparameter $\beta$, we perform an ablation study for the setting with inequal group sizes and $a = 0.3$. From Table \ref{tab:table11}, we can notice that a high $\beta$ 
leads to higher number of external iterations and a longer running time. Across the different $\beta$, the algorithm seems to find different local minima.  The best performances, in terms of test error and reconstruction errors, are reached with $\beta=0.5$, which we used for all the experiments. The stopping criterion is the condition $\dist_\infty(\theta^k, \{0,1\}^{P\times L})<tol$ with $tol=10^{-2}$ tolerance. The lower-level problem in~\eqref{prob:grouplasso} is solved using Algorithm~1 described in ~\citet{2018_Frecon_J_p-neurips_biglasso} stochastically, setting the batch size to $1$ as in their paper. Therefore, at each iteration, we consider one $w_t$, $\eta=10^{-3}$, $q=500$ inner iterations, and $0.99\,\eta / \lambda$ as the inner step size. For the upper-level optimization in \eqref{prob:approximate2}, we utilize SAGA~\citep{defazio2014saga} and 
we distribute the inner iterations for each external iteration as follows: $[5000, 5000, 2500, 2500, 2500, 2500, 2500, 1000, 1000, \ldots ]$. The step size is set to $T / 0.025 $ for $\theta$ and it is multiplied by the preconditioner $c=10^{-4}$ for $\lambda$. The hyperparameters $\theta$ and $\lambda$ are projected to the unit simplex  $(\Delta^{L-1})^P$ and the box $[10^{-3},1]$, respectively, and they are initialized to $\lambda^0=10^{-1}$ and $\theta^0 = \mathcal{P}_{\Theta}(L^{-1}\mathrm{I}_{P\times L} + \mathcal{N}(0_{P\times L},0.1L^{-1}\mathrm{I}_{P\times L} )$. 
For a fair comparison, in Figure~\ref{fig:1} and Tables~\ref{tab:table1} and \ref{tab:table2} we run all the methods with the same parameters and the same amount of total iterations. Here we report also
 in Figure \ref{fig:12} the evolution of the quantity $\dist_\infty (\theta^k, \Ybin)$ along the iterations for the \emph{relax and penalize} method, showing that for $k$ sufficiently large it abruptly reaches $0$ as soon as it is below $\frac{1}{2}$ as predicted in Remark \ref{remark3}.


\section{Details on the data distillation application}
\label{sec:DHC_auxiliary}

In this section, we report the calculations regarding the data distillation issue presented in Section~\ref{sec:DHC}.
Firstly, we present the calculations to solve exactly the lower-level
problem. Secondly, we report the calculations of the gradient of the
upper-level problem needed to perform the stochastic gradient
descent. Finally, we report the details of the experiments that were performed.

\subsection{Lower-level calculations}

The lower-level problem that we want to solve in \eqref{prob:20240607}
is given by
\begin{equation}
  \label{eq:1b}
  \min_{W,b} \quad
  \frac{1}{m} \sum_{i=1}^{m} v_i \norm{Wx_i^{\text{train}} + b -
    y_i^{\text{train}}}^2 + s\norm{W}^2
\end{equation}
with $x_i^{\text{train}} \in \mathbb{R}^d$, $y_i^{\text{train}} \in \mathbb{R}^e$,
$W \in \mathbb{R}^{e\times d}$, $b \in \mathbb{R}^e$.
In the subsequent sections, we will use $x$ and $y$ instead of
$x^{\text{train}}$ and $y^{\text{train}}$ for simplicity.
We define these quantities
\begin{equation*}
  \bar x = \frac{\sum_{i=1}^{m} v_i x_i}{\sum_{i=1}^{m} v_i},
  \quad \bar y = \frac{\sum_{i=1}^{m} v_i y_i}{\sum_{i=1}^{m} v_i},
  \quad \hat x_i = x_i - \bar x,
  \quad \hat y_i = y_i - \bar y.
\end{equation*}
It follows that
\begin{equation}
  \label{eq:zero}
  \sum_{i=1}^{m} v_i\hat x_i = \sum_{i=1}^{m} v_i x_i - \sum_{i=1}^{m}
  v_i  \bar x_i = 0,
  \quad \sum_{i=1}^{m} v_i\hat y_i = 0.
\end{equation}
Considering the first part of the objective function in \eqref{eq:1b},
we obtain
\begin{equation}
  \begin{split}
    \sum_{i=1}^{m} v_i \norm{Wx_i + b - y_i}^2
    &= \sum_{i=1}^{m} v_i \norm{W\hat x_i - \hat y_i  + W \bar x + b
      - \bar y}^2
    \\
    &=
    \sum_{i=1}^{m} v_i \norm{W\hat x_i -\hat y_i}^2 +
    \sum_{i=1}^{m} v_i \norm{W \bar x + b - \bar y}^2
    \\
    & \quad + 2 \sum_{i=1}^{m} v_i \langle W \sum_{i=1}^{m} v_i\hat
    x_i - \sum_{i=1}^{m} v_i\hat y_i , W \bar x + b - \bar y \rangle
    \\
    &=
    \sum_{i=1}^{m} v_i \norm{W \hat x_i - \hat y_i}^2 + \sum_{i=1}^{m}
    v_i \norm{W\bar x + b - \bar y}^2.
\end{split}
\end{equation}
Using this, Problem~\eqref{eq:1b} is equivalent to
\begin{equation*}
  \min_{W,b} \quad
  \frac{1}{m} \sum_{i=1}^{m} v_i \norm{W\hat x_i - \hat y_i}^2 + s\norm{W}^2 +
  \frac{1}{m} \left( \sum_{i=1}^{m} v_i \right) \norm{W\bar x + b - \bar y}^2.
\end{equation*}
Therefore, the minimization can be performed separately in the
variables~$w$ and~$b$, and Problem~\eqref{eq:1b} is equivalent to
\begin{equation*}
  \min_{W} \ \frac{1}{m} \sum_{i=1}^{m} v_i \norm{W\hat x_i - \hat
    y_i}^2 + s\norm{W}^2
  \quad \text{and} \quad b = \bar y - W \bar x.
\end{equation*}
We solve the first minimization in $W$ by setting
\begin{equation*}
  0 = \sum_{i=1}^m \frac{2}{m}v_i (W \hat x_i - \hat y_i) \hat x_i^\top + 2sW =
  2W \left( \frac{1}{m} \sum_{i=1}^m v_i \hat x_i \hat x_i^\top +s \text{Id} \right)
  - \frac{2}{m} \sum_{i=1}^m v_i \hat y_i \hat x_i^\top,
\end{equation*}
which is valid if and only if
\begin{equation*}
  W = \left(\frac{1}{m} \sum_{i=1}^m v_i  \hat y_i \hat x_i^\top \right)
  \left(\frac{1}{m}  \hat x_i \hat x_i^\top +s \text{Id}  \right)^{-1}.
\end{equation*}
Finally, we obtain the formula in \eqref{eq:ll}.
Computationally, we calculate $W(v)z$ for $z \in \mathbb{R}^d$ by
first solving
\begin{equation*}
  (C_v(X) + s \text{Id})a = z
\end{equation*}
in $a$ and then calculating
\begin{equation*}
  W(v)z = C_v(X,Y) a.
\end{equation*}

\subsection{Upper-level calculations}

  We consider the upper-level problem in \eqref{prob:20240607}
  neglecting the constraints on $v$:
  \begin{equation}
    \label{eq:2b}
    \min_{v \in \mathbb{R}^{m}} \quad
    \frac{1}{2n} \sum_{j=1}^{n} \norm{W(v)x^{\text{val}}_j + b(v) -
      y_j^{\text{val}}}^2.
  \end{equation}
  For simplicity, in the subsequent sections, we will use $x$ and $y$
  instead of $x^{\text{val}}$ and $y^{\text{val}}$, and we will refer to the
  objective function in \eqref{eq:2b} as $J(v)$ with $J: \mathbb{R}^n
  \rightarrow \mathbb{R}$.

To solve the problem with a stochastic gradient descent
  algorithm, we need to calculate the gradient of $J(v)$:
  \begin{equation}
    \begin{split}
      \frac{\partial J(v)}{\partial v_j}
      &=
      \frac{1}{n} \sum_{j=1}^{n}
      \left( W(v)x_i + b(v) - y_i\right)^\top
      \left( \frac{ \partial W (v)}{\partial v_j} x_i + \frac{\partial
          b(v)}{\partial v_j} \right)
      \\
      &=
      \frac{1}{n} \sum_{j=1}^{n}
      \left( W(v)   (x_i-\bar x_v) + \bar y_v - y_i\right)^\top
      \left( \frac{ \partial W (v)}{\partial v_j} (x_i - \bar x_v) +
        \frac{1}{\sum_{i=1}^m v_i} \hat y_j - W(v)\hat x_j \right)
      \\
      & =
      \frac{1}{n} \sum_{j=1}^{n}
      \left( W(v)   (x_i-\bar x_v) + \bar y_v - y_i\right)^\top
      \frac{ \partial W (v)}{\partial v_j} (x_i - \bar x_v)
      \\
      & \quad +
      \frac{1}{\sum_{i=1}^m v_i}
      \left( W (v) (x_i - \bar x_v) + \bar y_v - \bar y \right)^\top
      \left( \hat y_j - W(v) \hat x_j\right).
    \end{split}
  \end{equation}
  We will proceed to calculate $\frac{\partial W (v)}{\partial v_j}$,
  taking advantage of the expression given in~\eqref{eq:ll}.
We consider the following three maps
\begin{align*}
  \phi \colon \mathbb{R}^n \to \mathbb{R}^{n\times d},
  & \quad v \mapsto \sum_{i=1}^{n} v_i (y_i-\bar y_v) (x_i-\bar x_v)^\top,
  \\
  \psi \colon \mathbb{R}^n \to \mathbb{R}^{d\times d},
  & \quad v \mapsto \sum_{i=1}^{n} v_i (x_i-\bar x_v)
    (x_i-\bar x_v)^\top + s \text{Id},
  \\
  \varphi \colon \mathbb{GL}(d) \to \mathbb{R}^{d\times d},
  & \quad A \mapsto A^{-1}
\end{align*}
with $\mathbb{GL}(d)$ being the general linear group of degree $d$.
Therefore, we can write
\begin{equation}
  \label{eq:W}
  W(v) = \phi(v) \varphi(\psi(v)) \in \mathbb{R}^{n \times d}.
\end{equation}
Notice that
\begin{equation}
  \varphi'(A) \colon \mathbb{R}^{d\times d} \to \mathbb{R}^{d\times d},
  \quad U \mapsto A^{-1}UA^{-1} \quad \forall U \in \mathbb{R}^{d\times d}
\end{equation}
since
\begin{equation}
\label{eq:varphib}
\small
\frac{\varphi(A+tU) - \varphi(A)}{t} =
\frac{(A+tU)^{-1} - A^{-1}}{t} =
A^{-1} \frac{A-(A+tU)}{t} (A+tU)^{-1} =
A^{-1} U (A+tU)^{-1}.
\end{equation}
Using \eqref{eq:W} and \eqref{eq:varphib}, we can write
\begin{equation}
  \begin{split}
    \frac{W(v+tu) - W(v)}{t}
    &=
    \frac{\phi(v+tu) \varphi(\psi(v+tu)) - \phi(v)
      \varphi(\psi(v))}{t}
    \\
    &= \frac{1}{t} ( \phi(v+tu) \varphi(\psi(v+tu))  -\phi(v+tu) \varphi(\psi(v))
    \\
    & \quad +\phi(v+tu) \varphi(\psi(v)) -\phi(v) \varphi(\psi(v)) )
    \\
    &=
    \phi(v+tu)  \frac{\varphi(\psi(v+tu)) - \varphi(\psi(v))}{t} +
    \frac{\phi(v+tu) - \phi(v)}{t} \varphi(\psi(v)).
  \end{split}
\end{equation}
It follows that
\begin{equation}
  \begin{split}
    \label{eq:44}
    W'(v)[u]
    & =
    \phi(v) (\varphi \circ \psi)'(v)[u] + \phi'(v)[u] \varphi(\psi(v))
    \\
    &=
    \phi(v) \psi(v)^{-1} \psi'(v)[u] \psi(v)^{-1} +\phi'(v)[u]
    \psi(v)^{-1}
    \\
    &=
    (C_v(X,Y) (C_v(X)+s \text{Id})^{-1} \psi'(v)[u] + \phi'(v)[u])
    (C_v(X)+s \text{Id})^{-1} \in \mathbb{R}^{n\times d}
  \end{split}
\end{equation}
holds, where, in the second equality we used that
\begin{equation*}
  (\varphi \circ \psi)'(v)[u] =
  \varphi'(\psi(v)) (\psi'(v)[u]) =
  \psi(v)^{-1} \psi'(v)[u] \psi(v)^{-1}.
\end{equation*}
Choosing $u=e_j$ in \eqref{eq:44}, we obtain
\begin{equation}
\frac{\partial W(v)}{\partial v_j} =
(C_v(X,Y) (C_v(X)+s \text{Id})^{-1} \frac{\partial \psi(v)}{\partial v_j} + \frac{\partial \phi(v)}{\partial v_j}
(C_v(X)+s \text{Id})^{-1}.
\end{equation}
From the definitions of $\bar x_v$ and $\bar y_v$, we retrieve
\begin{align*}
  & \frac{\partial \bar x_v}{\partial v_j}
    = \frac{\partial}{\partial v_j} \left(
    \frac{\sum_{i=1}^{m}v_ix_i}{\sum_{i=1}^{m}v_i}\right) =
    -\frac{\sum_{i=1}^{m}v_ix_i}{\left(\sum_{i=1}^{m}v_i\right)^2}
    + \frac{x_j}{\sum_{i=1}^{m}v_i} =
    \frac{x_j - \bar x_v}{\sum_{i=1}^{m}v_i} =
    \frac{\hat x_j}{\sum_{i=1}^{m}v_i},
  \\
  & \frac{\partial \bar y_v}{\partial v_j} =
    \frac{y_j - \bar y_v}{\sum_{i=1}^{m}v_i} =
    \frac{\hat y_j}{\sum_{i=1}^{m}v_i}.
\end{align*}
Therefore,
\begin{align*}
  \frac{\partial \phi(v)}{\partial v_j}
  & = \sum_{i=1}^n \delta_{ij} (y_i-\bar y_v)(x_i-\bar x_v)^\top
    + \sum_{i=1}^n v_i \frac{\partial}{\partial v_j} (\bar y_v -
    y_i)(\bar x_v-x_i)^\top
  \\
  & = \hat y_j \hat x_j^\top - \frac{\hat y_j \sum_{i=1}^n v_i \hat x_i}{\sum_{i=1}^n v_i} -
    \frac{  \left(\sum_{i=1}^n v_i \hat y_i\right) \hat
    x_j^\top}{\sum_{i=1}^n v_i}
  \\
  & = \hat y_j \hat x_j^\top,
\end{align*}
where we used \eqref{eq:zero} in the last equation as well as
\begin{equation*}
  \frac{\partial \psi(v)}{\partial v_j} =
  \hat x_j \hat x_j^\top.
\end{equation*}
Finally,
\begin{equation*}
  \frac{\partial W (v)}{\partial v_j} =
  \left( C_v(X,Y) (C_v(X) + s \mathrm{Id})^{-1}
    \hat X_j \hat X_j^\top + \hat Y_j \hat X_j^\top \right) (C_v(X) +
  s \mathrm{Id})^{-1},
\end{equation*}
and
\begin{align*}
    \frac{\partial b (v)}{\partial v_j}
    &=
      \frac{\partial \bar y_v}{\partial v_j} -
      \frac{\partial W(v)}{\partial v_j} \bar X_v -
      W(v) \frac{\partial \bar x_v}{\partial v_j}
  \\
    &=
      \frac{\hat y_j}{\sum_{i=1}^n v_i} -
      \frac{\partial W(v)}{\partial v_j} \bar X_v -
      W(v) \frac{\hat x_j}{\sum_{i=1}^n v_i}
  \\
    &=
      \frac{\hat y_j - W(v)\hat x_j}{\sum_{i=1}^n v_i} -
      \frac{\partial W(v)}{\partial v_j} \hat x_v.
\end{align*}
As for the lower lever, computationally we calculate $\frac{\partial
  W(v)}{\partial v_j}z$ for $z \in \mathbb{R}^d$, solving the
following systems
\begin{equation*}
  (C_v(X)+s \text{Id})a _j = \hat x_j,
  \quad
  (C_v(X)+s \text{Id})a  = z.
\end{equation*}
After solving the upper-level problem with the stochastic gradient
descent method, we need to project the solution onto the simplex
defined by the knapsack constraint. To do this efficiently, we
utilize the Kiwiel algorithm \citep{kiwiel2008breakpoint}.

\subsection{Experimental setup}
\label{sec:app_D3}

In Section \ref{sec:DHC_exp}, we present experiments on the data distillation problem for two regression tasks involving the following real-world datasets:
  \begin{description}
  \item[\emph{music} \citep{year_prediction_msd_203}]
    is a dataset that includes song features from 1922 to 2011. It
    consists of 463\,715 training samples, with the first 231\,857 used
    for training the lower level and the remaining 231\,857 reserved for
    testing the weights afterward. Additionally, 51\,630 validation
    samples were utilized for the upper level.
    Each sample represents a song, featuring 90 attributes (12 related
    to timbre average and 78 related to timbre covariance), with the
    year of release as the target variable (as an integer).
    The aim is to predict the release year of a song based on its
    audio features.
  \item[\emph{blog} \citep{blogfeedback_304}] is a dataset
    containing features extracted from blog posts. It comprises 52\,397
    training and 7624 validation samples, with the first 1089
    used for training the lower level and the remaining 6535 set aside
    for testing the weights afterward. Each sample represents a post
    with 280 features, and the target variable is the number of
    comments received in the next 24 hours (as an integer).
    The goal is to predict comments received in the next 24 hours
    using various features.
  \end{description}
Regarding the details of the experiments, we implement the framework
of  Algorithm~\ref{alg:penalty-algorithm} by solving a sequence of
$K$ outer optimization problems \eqref{prob:refk} of type
\eqref{eq:approximate3} w.r.t.~$v$. The parameters are selected as explained in \ref{sec:GLS_setup}. In particular, we initialize $\varepsilon^0=10^9$ for both datasets, such that the penalty term does not dominate the function, making the evaluation at the first point behave similarly to the unconstrained version. In order to better exploit the hyperparameter $\beta$, we perform an ablation study for the setting \emph{perc}$= 20 \%$. From Table \ref{tab:table12}, we can notice that a high $\beta$ will create more thresholds until convergence, a higher leading to higher number of external iterations and a longer running time. Across the different $\beta$, the algorithm seems to find different local minima.  The best performances, in terms of test error and reconstruction errors, are reached with $\beta=0.9$, which we used for all the experiments. The true stopping criterion is the condition $\dist_\infty(v^k, \{0,1\}^m)<tol$ with $tol=10^{-2}$ tolerance.
For each problem of type \eqref{eq:approximate3}, we solve the lower-level problem exactly and the upper-level problem with stochastic gradient descent.
In the lower-level problem, we set the regularization parameter to $s=10^2$.
For the upper-level problem, we use a batch of size $600$ for computing time reasons, we perform $100$ inner iterations for each problem~\eqref{prob:refk}, and we set the step size to $10^{-3}$ for \emph{music} and to $10^{-5}$ for \emph{blog}.
The hyperparameter $v$ is projected onto the simplex defined by the
knapsack constraint and initialized at $\frac{\tau}{m}\mathrm{I}_m$
with~$m$ being the size of the training set and $\tau$ being the
budget.
For a fair comparison with the \emph{relaxation and rounding}, with both \emph{simple} and \emph{top-$k$ hard thresholding} rounding, we run the code with the same parameters and the same total number of iterations, rounding $v$ at the end.
We perform the projection onto the feasible set defined by the knapsack constraints using the Kiwiel algorithm with tolerance $10^{-10}$ and $10^{3}$ number of iterations.\\ 
We report in Figure \ref{fig:11} the evolution of the quantity $\dist_\infty (v^k, \Ybin)$ along the iterations for the \emph{relax and penalize} method, showing that for $k$ sufficiently large it rapidly reaches $0$ as soon as it is below $\frac{1}{2}$ as seen in Figure \ref{fig:12} for the group lasso problem.

\setlength{\tabcolsep}{2.5pt}
\begin{table}[t]
\caption{Ablation study hyperparameter $\beta$ for the setting \emph{perc}$= 20\%$.}
\label{tab:table12}
\centering
\begin{scriptsize}
\begin{NiceTabular}{l cccc cccc}
 & \multicolumn{4}{c}{\emph{music}} & \multicolumn{4}{c}{\emph{blog}}\\
\cmidrule(lr){2-5} \cmidrule(lr){6-9}
$\beta$ & outer iter & $\ell^{\text{val}}$ & $\ell^{\text{test}}$ & RMSE  & outer iter & $\ell^{\text{val}}$ & $\ell^{\text{test}}$ & RMSE\\
\midrule
0.1 & 14.2 $\pm$ 0.45 & 58.77 $\pm$ 0.12 & 60.59 $\pm$ 0.12 & 11.01 $\pm$ 0.01 & 15.2 $\pm$ 0.45 & 309.24 $\pm$ 0.53 & 234.94 $\pm$ 0.78 & 21.68 $\pm$ 0.04 \\ 
0.3 & 25.8 $\pm$ 0.45 & 58.16 $\pm$ 0.15 & 59.63 $\pm$ 0.16 & 10.92 $\pm$ 0.01 & 28.0 $\pm$ 0.00 & 306.75 $\pm$ 0.87 & 229.68 $\pm$ 2.39 & 21.43 $\pm$ 0.11 \\ 
0.5 & 43.0 $\pm$ 0.00 & 57.30 $\pm$ 0.09 & 58.61 $\pm$ 0.10 & 10.83 $\pm$ 0.01 & 47.0 $\pm$ 0.00 & 305.29 $\pm$ 0.34 & 225.26 $\pm$ 2.11 & 21.23 $\pm$ 0.10 \\ 
0.7 & 81.2 $\pm$ 0.45 & 56.03 $\pm$ 0.07 & 57.28 $\pm$ 0.08 & 10.70 $\pm$ 0.01 & 88.8 $\pm$ 0.84 & 302.71 $\pm$ 0.83 & 219.09 $\pm$ 1.81 & 20.93 $\pm$ 0.09 \\ 
$\pmb{0.9}$  & $\pmb{258.0 \pm 1.87}$  & $\pmb{54.23 \pm 0.03}$  & $\pmb{55.48 \pm 0.03}$  & $\pmb{10.53 \pm 0.00}$  & $\pmb{291.3 \pm 129.40}$  & $\pmb{296.77 \pm 25.52}$  & $\pmb{207.90 \pm 18.13}$  & $\pmb{20.39 \pm 0.85}$ \\
\bottomrule
\end{NiceTabular}
\end{scriptsize}
\vskip-1em
\end{table}

\begin{figure*}[htb]
  \centering
  \input{data/analysis}

\begin{tikzpicture}
\begin{groupplot}[
        group style={
            group name=dataset,
            columns=1,
            rows=1,
		},
        width = 5.5cm, 
        height = 5cm,
        axis x line*=bottom,
        xtick align=center,
        ytick align=center,
        ymajorgrids,
		grid style={draw=gray!50},
        y axis line style={draw=none},
		ytick style={draw=none},
        yticklabel style={xshift=-1em},
        every node near coord/.style={font=\scriptsize, yshift=-1.6em},
		y tick label style={
        /pgf/number format/.cd,
            fixed,
            precision=3,
        /tikz/.cd
    	},
        xticklabel style={rotate=35, anchor=east, yshift=-5pt, xshift=5pt, font=\footnotesize\sffamily},
        yticklabel style={font=\footnotesize\sffamily, xshift=10pt},
        ylabel style={font=\footnotesize\sffamily, align=center},
        xmin=0,xmax=30000,
        xtick distance=1000, 
        xticklabels={,0.1,,,,,0.5,,,,,1.0,,,,,1.5,,,,,2.0,,,,,2.5,,,,,3.0},
        x tick scale label style={yshift=10pt}
        ]
\nextgroupplot[]
    \addplot [thick, color=gray] table[x=k,y=distinf_list]{\pen};
\end{groupplot}
\end{tikzpicture}
  \caption{Evolution of $\dist_\infty (v^k, \Ybin)$ for the \emph{blog} dataset and budget 20\%.}
  \label{fig:11}
\end{figure*}

\subsection{Additional results along the steps}
\label{sec:DHC_auxiliary4}
 In Figure \ref{fig:2}-\ref{fig:4}, we report the evolution along the different iterations of the quantities reported in Table \ref{tab:table3}  for the \emph{music} dataset with distillation budget at $20\%$ of the training set size.
Since the solution can also violate integrality constraints through the iterations, the quantities $\left|D^{\text{train}}_{\text{syn}}\right|$, $\ell^{\text{test}}$, and RMSE, are referring to the correspondent rounded value.
The y-axis represents the internal iterations. Notice that, for example, the violation of the integrality constraints for the \emph{relax and penalize} method is not checked at each internal iteration, but only for the outer ones.
We recall that all the methods under analysis evolve with the same total number of iterations and start from the same point. 
Figure \ref{fig:2} and \ref{fig:3} shows the results regarding the \emph{relaxation and rounding} method, respectively using \emph{simple} and \emph{top-k hard thresholding} rounding. 
We show the evolution of $\ell^{\text{val}}$ and $\ell^{\text{test}}$ both before and after rounding the solution (resp. in orange and teal), noticing that the objective function values increase after rounding.
It can be observed that, for the \emph{relaxation and rounding} method with \emph{simple} rounding, the quantity $\left|D^{\text{train}}_{\text{syn}}\right|$ increases over the iterations but remains well below the total available budget (46371), while the \emph{top-k hard thresholding} use the whole budget at each iteration by definition. 
For the \emph{relaxation and rounding} method, $\dist_\infty (v, \Ybin)$ stays stable at $0.5$, therefore, no integer solution is found with both the rounding procedures. 
Figure \ref{fig:3} shows the results regarding the \emph{relax and penalize} method. 
From the plots of $\ell^{\text{val}}$, $\ell^{\text{test}}$ and RMSE, we can observe an increase after $25,000$ inner iterations, when the penalty term has a higher weight in \eqref{prob:refk}. 
Corresponding to this, it is apparent that the method finds an integer solution ($\dist_\infty (v, \Ybin) = 0$) and that it saturates the budget ($\left|D^{\text{train}}_{\text{syn}}\right| = 46371$).


\begin{figure*}[htb]
  \centering
  \input{data/analysis}

\begin{tikzpicture}

\begin{groupplot}[
        group style={
            group name=dataset,
            columns=3,
            rows=2,
            horizontal sep=1cm,
            vertical sep=1.2cm,
		},
        width = 5.5cm, 
        height = 5cm,
        axis x line*=bottom,
        xtick align=center,
        ytick align=center,
        ymajorgrids,
		grid style={draw=gray!50},
        y axis line style={draw=none},
		ytick style={draw=none},
        yticklabel style={xshift=-1em},
        every node near coord/.style={font=\scriptsize, yshift=-1.6em},
		y tick label style={
        /pgf/number format/.cd,
            fixed,
            precision=3,
        /tikz/.cd
    	},
        xticklabel style={rotate=35, anchor=east, yshift=-5pt, xshift=5pt, font=\footnotesize\sffamily},
        yticklabel style={font=\footnotesize\sffamily, xshift=10pt},
        ylabel style={font=\footnotesize\sffamily, align=center},
        xtick={0,5000,10000,15000,20000,25000,30000},
        xmin=-0.1,xmax=30000,
        x tick scale label style={yshift=10pt},
        y tick scale label style={yshift=-6pt},]

    \nextgroupplot[xticklabels={},scaled x ticks = false,title=(a) $\ell^{\text{val}}$,legend style={at={(0.5,0.45)},anchor=west,font=\small}]
    \addplot [thick, color=orange]table[x=k,y=of_list]{\rounding};
    \addlegendentry{before}
    \addplot [thick, color=teal]table[x=k,y=of_round_list]{\rounding};
    \addlegendentry{after}

    \nextgroupplot[xticklabels={},scaled x ticks = false,title=(b) $\left|D^{\text{train}}_{\text{syn}}\right|$,ymax=800]
    \addplot [thick, color=teal]table[x=k,y=rr_round_list]{\rounding};

    \nextgroupplot[title={(c) $\dist_\infty (v, \{0,1\}^m)$},ymax=0.55,ymin=0.35]
    \addplot [thick, color=orange]table[x=k,y=distbin_list]{\rounding};
 
    \nextgroupplot[ title=(d) $\ell^{\text{test}}$,legend style={at={(0.5,0.45)},anchor=west,font=\small}]
    \addplot [thick, color=orange]table[x=k,y=of_val_list]{\rounding};
    \addlegendentry{before}
    \addplot [thick, color=teal]table[x=k,y=of_round_val_list]{\rounding};
    \addlegendentry{after}

    \nextgroupplot[title= (e) RMSE,ymax=400,ymin=0]
    \addplot [thick, color=teal]table[x=k,y=RMSE_round_val_list]{\rounding};

\end{groupplot}

\end{tikzpicture}
  \caption{Evolution of the quantities reported in Table \ref{tab:table3} for the \emph{relaxation and rounding} method with \emph{simple} rounding on the \emph{music} dataset with distillation budget at $20\%$ of the training set size.}
  \label{fig:2}
\end{figure*}

\begin{figure*}[htb]
  \centering
  \input{data/analysis}

\begin{tikzpicture}

\begin{groupplot}[
        group style={
            group name=dataset,
            columns=3,
            rows=2,
            horizontal sep=1cm,
            vertical sep=1.2cm,
		},
        width = 5.5cm, 
        height = 5cm,
        axis x line*=bottom,
        xtick align=center,
        ytick align=center,
        ymajorgrids,
		grid style={draw=gray!50},
        y axis line style={draw=none},
		ytick style={draw=none},
        yticklabel style={xshift=-1em},
        every node near coord/.style={font=\scriptsize, yshift=-1.6em},
		y tick label style={
        /pgf/number format/.cd,
            fixed,
            precision=3,
        /tikz/.cd
    	},
        xticklabel style={rotate=35, anchor=east, yshift=-5pt, xshift=5pt, font=\footnotesize\sffamily},
        yticklabel style={font=\footnotesize\sffamily, xshift=10pt},
        ylabel style={font=\footnotesize\sffamily, align=center},
        xtick={0,5000,10000,15000,20000,25000,30000},
        xmin=-0.1,xmax=30000,
        x tick scale label style={yshift=10pt},
        y tick scale label style={yshift=-6pt},]

    \nextgroupplot[xticklabels={},scaled x ticks = false,title=(a) $\ell^{\text{val}}$,legend style={at={(0.5,0.45)},anchor=west,font=\small}]
    \addplot [thick, color=orange]table[x=k,y=of_list]{\rounding};
    \addlegendentry{before}
    \addplot [thick, color=teal]table[x=k,y=of_round_list_new]{\rounding};
    \addlegendentry{after}

    \nextgroupplot[xticklabels={},scaled x ticks = false,title=(b) $\left|D^{\text{train}}_{\text{syn}}\right|$,ymin=0,ymax=11000]
    \addplot [thick, color=teal]table[x=k,y=rr_round_list_new]{\rounding};

    \nextgroupplot[title={(c) $\dist_\infty (v, \{0,1\}^m)$},ymax=0.55,ymin=0.35]
    \addplot [thick, color=orange]table[x=k,y=distbin_list]{\rounding};
 
    \nextgroupplot[title=(d) $\ell^{\text{test}}$,legend style={at={(0.5,0.45)},anchor=west,font=\small}]
    \addplot [thick, color=orange]table[x=k,y=of_val_list]{\rounding};
    \addlegendentry{before}
    \addplot [thick, color=teal]table[x=k,y=of_round_val_list_new]{\rounding};
    \addlegendentry{after}

    \nextgroupplot[title= (e) RMSE,ymax=30,ymin=20]
    \addplot [thick, color=teal]table[x=k,y=RMSE_round_val_list_new]{\rounding};

\end{groupplot}

\end{tikzpicture}
  \caption{Evolution of the quantities reported in Table \ref{tab:table3} for the \emph{relaxation and rounding} method with \emph{top-k hard thresholding} rounding on the \emph{music} dataset with distillation budget at $20\%$ of the training set size.}
  \label{fig:3}
\end{figure*}

\begin{figure*}[htb]
  \centering
  \input{data/analysis}

\definecolor{pink2}{RGB}{216,166,166}
\begin{tikzpicture}

\begin{groupplot}[
        group style={
            group name=dataset,
            columns=3,
            rows=2,
            horizontal sep=1cm,
            vertical sep=1.2cm,
		},
        width = 5.5cm, 
        height = 5cm,
        axis x line*=bottom,
        xtick align=center,
        ytick align=center,
        ymajorgrids,
		grid style={draw=gray!50},
        y axis line style={draw=none},
		ytick style={draw=none},
        yticklabel style={xshift=-1em},
        every node near coord/.style={font=\scriptsize, yshift=-1.6em},
		y tick label style={
        /pgf/number format/.cd,
            fixed,
            precision=3,
        /tikz/.cd
    	},
        xticklabel style={rotate=35, anchor=east, yshift=-5pt, xshift=5pt, font=\footnotesize\sffamily},
        yticklabel style={font=\footnotesize\sffamily, xshift=10pt},
        ylabel style={font=\footnotesize\sffamily, align=center},
        xtick={0,5000,10000,15000,20000,25000,30000},
        xmin=-0.1,xmax=30000,
        x tick scale label style={yshift=10pt},
        y tick scale label style={yshift=-6pt},]

    \nextgroupplot[xticklabels={},scaled x ticks = false,title=(a) $\ell^{\text{val}}$]
    \addplot [thick, color=purple]table[x=k,y=of_list]{\pen};

    \nextgroupplot[xticklabels={},scaled x ticks = false,title=(b) $\left|D^{\text{train}}_{\text{syn}}\right|$],ymax=60000]
    \addplot [thick, color=purple]table[x=k,y=rr_round_list]{\pen};
    \draw [draw=black, line width=0.05cm, densely dotted] (0,10479) -- (30000,10479);
    \node at (1cm, 2.7cm) {budget};

    \nextgroupplot[xticklabels={},scaled x ticks = false,title={(c) $\dist_\infty (v, \{0,1\}^m)$},ymax=0.6]
    \addplot [thick, color=purple]table[x=k,y=distbin_list]{\pen};
 
    \nextgroupplot[title=(d) $\ell^{\text{test}}$]
    \addplot [thick, color=purple]table[x=k,y=of_val_list]{\pen};

    \nextgroupplot[title= (e) RMSE,ymax=11.5,ymax=25]
    \addplot [thick, color=purple]table[x=k,y=RMSE_val_list]{\pen};

    
\end{groupplot}

\end{tikzpicture}
  \caption{Evolution of the quantities reported in Table \ref{tab:table3} for the \emph{relax and penalize} method on the \emph{music} dataset with distillation budget at $20\%$ of the training set size.}
  \label{fig:4}
\end{figure*}


\end{document}